\documentclass[10pt]{article}
\usepackage{amsmath,graphicx}
\usepackage{amsthm}
\newtheorem{theorem}{Theorem}
\usepackage{algorithm,algpseudocode}
\newtheorem{prop}{Proposition}
\usepackage{url}
\usepackage{booktabs}

\usepackage{epstopdf}
\usepackage{amssymb}
\usepackage{tikz}
\usepackage{pbox}

\def\vect{\mathbf}
\def\matr{\mathbf}
\def\mrm{\mathrm}

\newcommand{\tr}{\mathop{\mathrm{tr}}\nolimits}

\DeclareMathOperator*{\argmin}{arg\,min}
\newcommand{\inprod}[2]{\langle #1 , #2 \rangle}

\usepackage{cite}
\usepackage{hyperref}

\newcommand{\cape}{\mathsf{CAPE}}
\newcommand{\capefm}{\mathsf{capeFM}}
\newcommand{\dpfm}{\mathsf{dp-fm}}
\newcommand{\objpert}{\mathsf{objPert}}
\newcommand{\conv}{\mathsf{conv}}
\newcommand{\local}{\mathsf{local}}

\newcommand{\secureagg}{\mathsf{SecureAgg}}

\newcommand{\nonpriv}{\mathsf{Non-priv.}}
\newcommand{\nonprivT}{\mathsf{Non-priv.}}

\usepackage{mathtools}
\DeclarePairedDelimiter{\ceil}{\lceil}{\rceil}

\usepackage[normalem]{ulem}
\newcommand\redout{\bgroup\markoverwith{\textcolor{red}{\rule[.5ex]{2pt}{0.4pt}}}\ULon}

\newtheorem{lemma}{Lemma}
\newtheorem{Def}{Definition}

\newtheorem{Rem}{Remark}
\begin{document}

\title{Distributed Differentially Private Computation of Functions with Correlated Noise}
\date{}
\author{Hafiz Imtiaz, %
Jafar Mohammadi, and %
Anand D. Sarwate%
\thanks{H.~Imtiaz and A.D.~Sarwate are with the Department of Electrical and Computer Engineering, Rutgers, The State University of New Jersey, 94 Brett Road, Piscataway, NJ 08854, USA. Email: \texttt{hafiz.imtiaz@rutgers.edu}, \texttt{anand.sarwate@rutgers.edu}. J.~Mohammadi is with Nokia Bell Labs, Lorenzstrasse 10, 70435 Stuttgart, Germany. Email: \texttt{jmohammadi@jacobs-alumni.de}.}
\thanks{This work was supported in part by the United States National Institutes of Health under award 1R01DA040487, the United States National Science Foundation under award CCF-1453432, and DARPA and SSC Pacific under contract N66001-15-C-4070.}%
\thanks{Manuscript received January xx, 2019; revised xxxx xx, 2019.}
}

%

\maketitle

\begin{abstract}
Many applications of machine learning, such as human health research, involve processing private or sensitive information. Privacy concerns may impose significant hurdles to collaboration in scenarios where there are multiple sites holding data and the goal is to estimate properties jointly across all datasets. Differentially private decentralized algorithms can provide strong privacy guarantees. However, the accuracy of the joint estimates may be poor when the datasets at each site are small. This paper proposes a new framework, Correlation Assisted Private Estimation ($\cape$), for designing privacy-preserving decentralized algorithms with better accuracy guarantees in an honest-but-curious model. $\cape$ can be used in conjunction with the functional mechanism for statistical and machine learning optimization problems. A tighter characterization of the functional mechanism is provided that allows $\cape$ to achieve the same performance as a centralized algorithm in the decentralized setting using all datasets. Empirical results on regression and neural network problems for both synthetic and real datasets show that differentially private methods can be competitive with non-private algorithms in many scenarios of interest.
\end{abstract}


\section{Introduction}\label{sec:intro}
Privacy-sensitive learning is important in many applications: examples include human health research, business informatics, and location-based services among others. Releasing any function of private data, even summary statistics and other aggregates, can reveal information about the underlying training data. Differential privacy (DP)~\cite{dwork2006} is a cryptographically motivated and mathematically rigorous framework for measuring the risk associated with performing computations on private data. More specifically, it measures the privacy risk in terms of the probability of identifying the presence of individual data points in a dataset from the results of computations performed on that data. As such, it has emerged as a de-facto standard for privacy-preserving technologies in research and practice~\cite{rappor,appleDP,UScensus}.

Differential privacy is also useful when the private data is distributed over different locations (sites). For example, a consortium for medical research on a particular disease may consist of several healthcare centers/research labs, each with their own dataset of human subjects~\cite{enigma,coinstac}. Data holders may be reluctant or unable to directly share ``raw'' data to an aggregator due to ethical (privacy) and technical (bandwidth) reasons. From a statistical standpoint, the number of samples held locally is usually not large enough for meaningful feature learning. Consider training a deep neural network to detect Alzheimer's disease based on neuroimaging data from several studies~\cite{coinstac}: training locally at one site is infeasible as the number of subjects in each study is small. Decentralized algorithms can allow data owners to maintain local control of the data while passing messages to assist in a joint computation across many datasets. If these computations are differentially private, they can measure and control privacy risks. 

Differentially private algorithms introduce noise to guarantee privacy: conventional distributed DP algorithms often have poor utility due to excess noise compared to centralized analyses. In this paper we propose a Correlation Assisted Private Estimation ($\cape$) framework, which is a novel  \emph{distributed and privacy-preserving} protocol that provides utility close to centralized case. We achieve this by inducing (anti) correlated noise in the differentially private messages. The $\cape$ protocol can be applied to computing loss functions that are separable across sites. This class includes optimization algorithms, such as empirical risk minimization (ERM) problems, common in machine learning (ML) applications.

\noindent\textbf{Related Works.} 
There is a vast literature~\cite{boyd2011, molzahn2017, uribe2017, nedic2009, athans1986, huang2015, han2017, nozari2016, zhu2018} on solving optimization problems in distributed settings, both with and without privacy concerns. In the machine learning context, the most relevant ones to our current work are those using ERM and stochastic gradient descent (SGD)~\cite{Kamalika09, anand2011, anandSGD, abadi2016, Lipton14, DPonline18, bassily2014, Katrina17, Wang17}. Additionally, several works studied distributed differentially private learning for locally trained classifiers~\cite{Pathak10,Agarwal12,BassilyKobbi17}. One of the most common approaches for ensuring differential privacy in optimization problems is to employ randomized gradient computations~\cite{anandSGD, bassily2014}. Another common approach is to employ the output perturbation~\cite{anand2011}, which adds noise to the output of the optimization problem according to the sensitivity of the optimization variable. Note that, both of these approaches involve computing the sensitivity (of the gradient or the output variable) and then adding noise scaled to the sensitivity~\cite{dwork2006}. The problem with output perturbation is that the relation between the data and the parameter set is often hard to characterize. This is due to the complex nature of the optimization and as a result, the sensitivity is very difficult to compute. However, differentially-private gradient descent methods can circumvent this by bounding the gradients at the expense of slowing down the convergence process. Finally, one can employ the objective perturbation~\cite{anand2011, nozari2016}, where we need to perturb the objective function and find the minimizer of the perturbed objective function. However, the objective function has to satisfy some strict conditions, which are not met in many practical optimization problems~\cite{zhang2012}. In addition to optimization problems, Smith~\cite{smith2011} proposed a general approach for computing summary statistics using the \emph{sample-and-aggregate} framework and both the Laplace and Exponential mechanisms~\cite{mcsherry2007}. Jing~\cite{jing2011} proposed a unique approach that uses perturbed histograms for releasing a class of $M$-estimators in a non-interactive way. 

Differentially private algorithms provide different guarantees than Secure Multi-party Computation (SMC) based methods (see ~\cite{SlawomirMPC17, ShoukryPappas16, Eigner14, Kairouz2015smc, Suresh16, Bonawitz16} for thorough comparisons between SMC and differential privacy based methods). Gade and Vaidya~\cite{GadeSMC16} applied a combination of SMC and DP for distributed optimization in which each site adds and subtracts arbitrary functions to confuse the adversary. Bonawitz et al.~\cite{Bonawitz17} proposed a communication-efficient method for \emph{federated learning} over a large number of mobile devices. The most recent work in this line is that of Heikkil\"{a} et al.~\cite{Mikko17}, who also studied the relationship of additive noise and sample size in a distributed setting. In their model, $S$ data holders communicate their data to $M$ computation nodes to compute a function. Our work is inspired by the seminal work of Dwork et al.~\cite{Ourdata06} that proposed distributed noise generation for preserving privacy. We employ a similar principle as Anandan and Clifton~\cite{Clifton15} to \emph{reduce} the noise added for differential privacy. 

Our application to distributed DP function computation in this paper builds on the \emph{functional mechanism}~\cite{zhang2012}, which uses functional approximation~\cite{rudin1976} and the Laplace mechanism~\cite{dwork2006} to create DP approximations for any continuous and differentiable function. Zhang et al.'s approach~\cite{zhang2012} does not scale well to decentralized problems. We provide a better analysis of the sensitivity of their approximation and adapt the approach to the decentralized setting.

\begin{figure}[t]
  \centering
  \includegraphics[width=0.8\columnwidth]{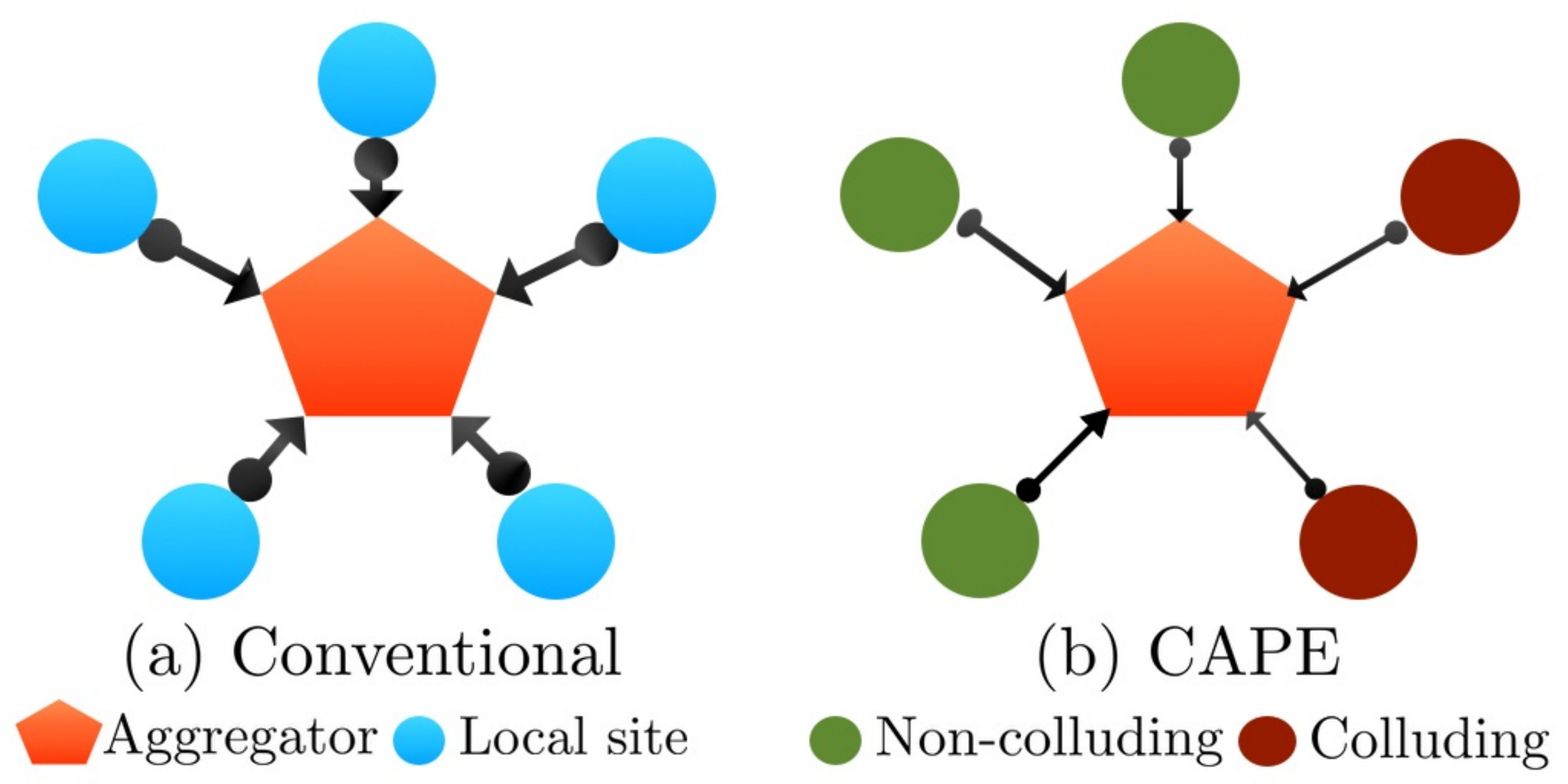}\\
  \vspace{-0.0in}
  \caption{The structure of the network: (a) conventional, (b) $\cape$}
  \label{fig:network_structure}
\end{figure}

\noindent\textbf{Our Contribution.} 
The goal of our work is to reduce the amount of noise in conventional distributed differentially private schemes for applications of machine learning to settings similar to those found in research consortia. 
We summarize our contributions here:
\begin{itemize}
\item We propose a novel distributed computation protocol, $\cape$, that improves upon the conventional distributed DP schemes and achieves the same level of utility as the pooled data scenario in certain regimes. $\cape$ can be employed in a wide range of computations that frequently appear in machine learning problems.
\item We propose an improved functional mechanism (FM) using a tighter sensitivity analysis. We show analytically that it guarantees less noisy function computation for linear and logistic regression problems at the expense of an approximate DP guarantee. Empirical validation on real and synthetic data validates our approach.
\item We extend the FM to decentralized settings and show that $\cape$ can achieve the same utility as the pooled data scenario in some regimes. To the best of our knowledge, this work proposes the first distributed functional mechanism. 
\item We demonstrate the effectiveness of our algorithms with varying privacy and dataset parameters. Our privacy analysis and empirical results on real and synthetic datasets show that the proposed algorithms can achieve much better utility than the existing state of the art algorithms.
\end{itemize}
Note that, we showed a preliminary version of the $\cape$ protocol in~\cite{imtiaz2018}. The protocol in this paper is more robust against site dropouts and does not require a trusted third-party.

\section{Data and Privacy Model}\label{sec:problem_formulation}
\noindent\textbf{Notation.} We denote vectors with bold lower case letters (e.g., $\vect{x}$), matrices with bold upper case letters (e.g. $\matr{X}$), scalars with regular letters (e.g., $M$) and indices with lower case letters (e.g., $m$). Indices typically run from 1 to their upper-case versions (e.g., $m \in \{1, 2, \ldots, M\} \triangleq [M]$). We denote the $n$-th column of the matrix $\matr{X}$ as $\vect{x}_n$. We use $\|\cdot\|_2$,  $\|\cdot\|_F$ and $\tr(\cdot)$ for the Euclidean (or $\mathcal{L}_2$) norm of a vector or spectral norm of a matrix, the Frobenius norm, and the trace operation, respectively. Finally, we denote the inner-product between two arrays as $\inprod{\cdot}{\cdot}$. For example, if $\matr{A}$ and $\matr{B}$ are two matrices then $\inprod{\matr{A}}{\matr{B}} = \tr\left(\matr{A}^\top\matr{B}\right)$.

\noindent\textbf{Distributed Data Setting.} We consider a distributed data setting with $S$ sites and a central aggregator node (see Figure \ref{fig:network_structure}). Each site $s \in [S]$ holds $N_s$ samples and the total number of samples across all sites is given by $N = \sum_{s=1}^S N_s$. We assume that all parties are ``honest but curious''. That is, the sites and the aggregator will follow the protocol but a subset may collude to learn another site's data/function output. Additionally, we assume that the data samples in the local sites are disjoint. We use the terms ``distributed'' and ``decentralized'' interchangeably in this paper. 

\begin{Def}[($\epsilon, \delta$)-Differential Privacy~\cite{dwork2006}] 
An algorithm $\mathcal{A}(\mathbb{D})$ taking values in a set $\mathbb{T}$ provides $(\epsilon,\delta)$-differential privacy if $\Pr[\mathcal{A}(\mathbb{D}) \in \mathbb{S}] \leq \exp(\epsilon) \Pr[\mathcal{A}(\mathbb{D'}) \in \mathbb{S}] + \delta$, for all measurable $\mathbb{S} \subseteq \mathbb{T}$ and all data sets $\mathbb{D}$ and $\mathbb{D'}$ differing in a single entry (neighboring datasets).
\end{Def}
\noindent This definition essentially states that the probability of the output of an algorithm is not changed significantly if the corresponding database input is changed by just one entry. Here, $\epsilon$ and $\delta$ are privacy parameters, where low $\epsilon$ and $\delta$ ensure more privacy. The parameter $\delta$ can be interpreted as the probability that the algorithm fails to provide privacy risk $\epsilon$. Several mechanisms can be employed to ensure that an algorithm satisfies differential privacy. Additive noise mechanisms such as the Gaussian or Laplace mechanisms~\cite{dwork2006, dwork2013algorithmic} and random sampling using the exponential mechanism~\cite{mcsherry2007} are among the most common ones. For additive noise mechanisms, the standard deviation of the noise is scaled to the \emph{sensitivity} of the computation. 

\begin{Def}[$\mathcal{L}_p$-sensitivity~\cite{dwork2006}]
The $\mathcal{L}_p$-sensitivity of a vector-valued function $f(\mathbb{D})$ is $\Delta := \max_{ \mathbb{D}, \mathbb{D'} } \|f(\mathbb{D})-f(\mathbb{D'})\|_p$,  where $\mathbb{D}$ and $\mathbb{D'}$ are neighboring datasets. 
\end{Def}
\noindent We will focus on $p=1$ and $2$ in this paper.

\begin{Def}[Gaussian Mechanism~\cite{dwork2013algorithmic}]
Let $f: \mathbb{D} \mapsto \mathbb{R}^D$ be an arbitrary $D$-dimensional function with $\mathcal{L}_2$-sensitivity $\Delta$. The Gaussian Mechanism with parameter $\tau$ adds noise scaled to $\mathcal{N}(0, \tau^2)$ to each of the $D$ components of the output and satisfies $(\epsilon, \delta)$ differential privacy if
\begin{align}
\tau &\geq \frac{\Delta}{\epsilon}\sqrt{2\log\frac{1.25}{\delta}}.
\label{eq:GaussMech}
\end{align}
\end{Def}
\noindent Note that, for any given $(\epsilon, \delta)$ pair, we can calculate a noise variance $\tau^2$ such that addition of a noise term drawn from $\mathcal{N}(0, \tau^2)$ guarantees $(\epsilon, \delta)$-differential privacy. Since there are infinitely many $(\epsilon, \delta)$ pairs that yield the same $\tau^2$, we parameterize our methods using $\tau^2$~\cite{imtiaz2018} in this paper.

\section{Correlation Assisted Private Estimation}\label{sec:cape}

\subsection{Conventional Approach to Distributed DP Computations} \label{sec:dist_avg}
We now describe the problem with conventional distributed DP and the $\cape$ approach to improve performance~\cite{imtiaz2018}.
Suppose we want to compute the average of $N$ data samples. Each sample $x_n$ is a scalar with $x_n \in \left[0,1\right]$. We denote the vector of $N$ data samples as $\vect{x} = \left[x_1,\ldots, x_{N-1},\ x_N\right]^\top$. We are interested in computing the $(\epsilon, \delta)$-DP estimate of the mean function: $f(\vect{x}) = \frac{1}{N} \sum_{n = 1}^N x_n$. To compute the sensitivity~\cite{dwork2006} of the scalar-valued function $f(\vect{x})$, we consider a neighboring data vector $\vect{x'} = \left[x_1,\ldots, x_{N-1},\ x'_N\right]^\top$. We observe $\left|f(\vect{x}) - f(\vect{x'})\right| = \frac{1}{N}\left|x_N - x'_N\right| \leq \frac{1}{N}$, which follows from the assumption $x_n \in \left[0,1\right]$. Therefore, to compute the $(\epsilon, \delta)$-DP estimate of the average $a = f(\vect{x})$, we can employ the Gaussian mechanism~\cite{dwork2006, dwork2013algorithmic} to release $\hat{a} = a + e$, where $e \sim \mathcal{N}\left(0, \tau^2\right)$ and $\tau = \frac{1}{N\epsilon}\sqrt{2\log \frac{1.25}{\delta}}$.

Each site $s$ holds $N_s$ samples $\vect{x}_s \in \mathbb{R}^{N_s}$ (see Figure \ref{fig:network_structure}(a)). We assume $N_s = \frac{N}{S}$ for simplicity. To compute the global average non-privately, the sites can send $a_s = f(\vect{x}_s)$ to the aggregator and the average computed by aggregator ($a_\mrm{conv} = \frac{1}{S}\sum_{s=1}^S a_s$) is exactly equal to the average we would get if all the data samples were available in the aggregator node. However, with the privacy concern and considering that the aggregator is honest-but-curious, the sites can employ the conventional distributed DP computation technique. That is, the sites will release (send to the aggregator node) an $(\epsilon, \delta)$-DP estimate of the function $f(\vect{x}_s)$ of their local data $\vect{x}_s$. More specifically, each site will generate a noise $e_s \sim \mathcal{N}\left(0, \tau_s^2\right)$ and release/send $\hat{a}_s = f(\vect{x}_s) + e_s$ to the aggregator, where 
\begin{align*}
\tau_s &= \frac{1}{N_s\epsilon}\sqrt{2\log \frac{1.25}{\delta}} = \frac{S}{N\epsilon}\sqrt{2\log \frac{1.25}{\delta}}.
\end{align*}
The aggregator can then compute the $(\epsilon, \delta)$-DP approximate average as $a_\mrm{conv} = \frac{1}{S}\sum_{s=1}^S \hat{a}_s$. We observe
\begin{align*}
a_\mrm{conv} &= \frac{1}{S}\sum_{s=1}^S \hat{a}_s = \frac{1}{S}\sum_{s=1}^S a_s + \frac{1}{S}\sum_{s=1}^S e_s.
\end{align*}
The variance of the estimator $a_\mrm{conv}$ is $S \cdot \dfrac{\tau_s^2}{S^2} = \dfrac{\tau_s^2}{S} \triangleq \tau_\mrm{conv}^2$. However, if we had all the data samples at the aggregator (pooled-data scenario), we could compute the $(\epsilon, \delta)$-DP estimate of the average as $a_\mrm{pool} = \frac{1}{N}\sum_{n=1}^N x_n + e_\mrm{pool}$, where $e_\mrm{pool} \sim \mathcal{N}\left(0, \tau_\mrm{pool}^2\right)$ and $\tau_\mrm{pool} = \frac{1}{N\epsilon}\sqrt{2\log \frac{1.25}{\delta}} = \frac{\tau_s}{S}$. We observe the ratio
\begin{align*}
\frac{\tau_\mrm{pool}^2}{\tau_\mrm{conv}^2} &= \frac{\frac{\tau_s^2}{S^2}}{\frac{\tau_s^2}{S}} = \frac{1}{S}.
\end{align*}
That is, the distributed DP averaging scheme will always result in a worse performance than the DP pooled data case. 

\subsection{Proposed Scheme: $\cape$}\label{sec:cape_details}
\begin{algorithm}[t] 
	\caption{Generate zero-sum noise \label{alg:zero-sum-noise-generation}}
	\begin{algorithmic}[1]
    \Require 
    Local noise variances $\{\tau_s^2\}$; security parameter $\lambda$, threshold value $t$
    \State Each site generate $\hat{e}_s \sim \mathcal{N}(0, \tau_s^2)$
    \State Aggregator computes $\sum_{s=1}^S \hat{e}_s$ according to $\secureagg(\lambda, t)$~\cite{Bonawitz17}
    \State Aggregator broadcasts $\sum_{s=1}^S \hat{e}_s$ to all sites $s \in \{1, \ldots, S\}$
    \State Each site computes $e_s = \hat{e}_s - \frac{1}{S}\sum_{s'=1}^S \hat{e}_{s'}$\\
    \Return $e_s$
    \end{algorithmic}
\end{algorithm}

\noindent\textbf{Trust/Collusion Model. }In our proposed $\cape$ scheme, we assume that all of the $S$ sites and the central node follow the protocol honestly. However, up to $S_C = \ceil*{\frac{S}{3}} - 1$ sites can collude with an adversary to learn about some site's data/function output. The central node is also honest-but-curious (and therefore, can collude with an adversary). An adversary can observe the outputs from each site, as well as the output from the aggregator. Additionally, the adversary can know everything about the colluding sites (including their private data). We denote the number of non-colluding sites with $S_H$ such that $S = S_C + S_H$. Without loss of generality, we designate the non-colluding sites with $\{1, \ldots, S_H\}$ (see Figure \ref{fig:network_structure}(b)).

\noindent\textbf{Correlated Noise. }We design the noise generation procedure such that: i) we can ensure $(\epsilon, \delta)$ differential privacy of the algorithm output from each site and ii) achieve the noise level of the pooled data scenario in the final output from the aggregator. We achieve that by employing a correlated noise addition scheme. Considering the same distributed averaging problem as Section \ref{sec:dist_avg}, we intend to release (and send to the aggregator) $\hat{a}_s = f(\vect{x}_s) + e_s + g_s$ from each site $s$, where $e_s$ and $g_s$ are two noise terms. The variances of $e_s$ and $g_s$ are chosen to ensure that the noise $e_s + g_s$ is sufficient to guarantee $(\epsilon, \delta)$-differential privacy to $f(\vect{x}_s)$. Here, each site generates the noise $g_s \sim \mathcal{N}(0, \tau_g^2)$ locally and the noise $e_s \sim \mathcal{N}(0, \tau_e^2)$ jointly with all other sites such that $\sum_{s=1}^S e_s = 0$. We employ the recently proposed secure aggregation protocol $(\secureagg)$ by Bonawitz et al.~\cite{Bonawitz17} to generate $e_s$ that ensures $\sum_{s=1}^S e_s = 0$. The $\secureagg$ protocol utilizes Shamir's $t$-out-of-$n$ secret sharing~\cite{shamir1979} and is communication-efficient. 

\noindent\textbf{Detailed Description of $\cape$ Protocol. }In our proposed scheme, each site $s \in [S]$ generates a noise term $\hat{e}_s \sim \mathcal{N}(0, \tau_s^2)$ independently. The aggregator computes $\sum_{s=1}^S \hat{e}_s$ according to the $\secureagg$ protocol and broadcasts it to all the sites. Each site then sets $e_s = \hat{e}_s - \frac{1}{S}\sum_{s'=1}^S \hat{e}_{s'}$ to achieve $\sum_{s=1}^S e_s = 0$. We show the complete noise generation procedure in Algorithm \ref{alg:zero-sum-noise-generation}. Note that, the original $\secureagg$ protocol is intended for computing sum of $D$-dimensional vectors in a finite field $\mathbb{Z}^D_\lambda$. However, we need to perform the summation of Gaussian random variables over $\mathbb{R}$ or $\mathbb{R}^D$. To accomplish this, each site can employ a mapping $\mrm{map}: \mathbb{R} \mapsto \mathbb{Z}_\lambda$ that performs a  stochastic quantization~\cite{salman2019} for large-enough $\lambda$. The aggregator can compute the sum in the finite field according to $\secureagg$ and then invoke a reverse mapping $\mrm{remap}: \mathbb{Z}_\lambda \mapsto \mathbb{R}$ before broadcasting $\sum_{s=1}^S \hat{e}_s$ to the sites. Algorithm \ref{alg:zero-sum-noise-generation} can be readily extended to generate array-valued zero-sum noise terms. We observe that the variance of $e_s$ is given by
\begin{align}\label{eqn:cape_noise_variance_e}
    \tau_e^2 &= \mathbb{E}\left[\left(\hat{e}_s - \frac{1}{S}\sum_{s'=1}^S \hat{e}_{s'}\right)^2\right] = \left(1-\frac{1}{S}\right)\tau^2_s.
\end{align}
Additionally, we choose
\begin{align}\label{eqn:cape_noise_variance_g}
\tau_g^2 = \frac{\tau^2_s}{S}.
\end{align}
Each site then generates the noise $g_s \sim \mathcal{N}(0, \tau_g^2)$ independently and sends $\hat{a}_s = f(\vect{x}_s) + e_s + g_s$ to the aggregator. Note that neither of the terms $e_s$ and $g_s$ has large enough variance to provide $(\epsilon, \delta)$-DP guarantee to $f(\vect{x}_s)$. However, we chose the variances of $e_s$ and $g_s$ to ensure that the $e_s + g_s$ is sufficient to ensure a DP guarantee to $f(\vect{x}_s)$ at site $s$. The chosen variance of $g_s$ also ensures that the output from the aggregator would have the same noise variance as the differentially private pooled-data scenario. To see this, observe that we compute the following at the aggregator (in Step \ref{alg:dp_avg:step_a_ag} of Algorithm \ref{alg:dp_avg}): 
\begin{align*}
a_\mrm{cape} 	&= \frac{1}{S} \sum_{s=1}^S \hat{a}_s = \frac{1}{S} \sum_{s=1}^S f(\vect{x}_s) + \frac{1}{S} \sum_{s=1}^S g_s,
\end{align*}
where we used $\sum_s e_s = 0$. The variance of the estimator $a_\mrm{cape}$ is $\tau^2_\mrm{cape} = S \cdot \frac{\tau_g^2}{S^2} = \tau^2_\mrm{pool}$, which is the exactly the same as if all the data were present at the aggregator. This claim is formalized in Lemma~\ref{lemma:cape}. We show the complete algorithm in Algorithm \ref{alg:dp_avg}. The privacy of Algorithm \ref{alg:dp_avg} is given by Theorem \ref{thm:cape}. The communication cost of the $\cape$ scheme is discussed in Appendix~\ref{appendix:cape_comm} in the Supplement.

\begin{algorithm}[t] 
	\caption{Correlation Assisted Private Estimation ($\cape$)\label{alg:dp_avg}}
	\begin{algorithmic}[1]
    \Require Data samples $\{\vect{x}_s\}$, 
    local noise variances $\{\tau_s^2\}$
    \For{$s = 1,\ \ldots,\ S$} \Comment{at each site}
        \State Generate $e_s$ according to Algorithm \ref{alg:zero-sum-noise-generation}
        \State Generate $g_s \sim \mathcal{N}(0, \tau_g^2)$ with $\tau_g^2 = \frac{\tau_s^2}{S}$
        \State Compute and send $\hat{a}_s \gets f(\vect{x}_s) + e_s + g_s$ \label{alg:dp_avg:step_as_hat}
    \EndFor
    \State Compute $a_\mrm{cape} \gets \frac{1}{S} \sum_{s=1}^S \hat{a}_s$ \label{alg:dp_avg:step_a_ag} \Comment{at the aggregator}\\
    \Return $a_\mrm{cape}$
    \end{algorithmic}
\end{algorithm}

\begin{theorem}[Privacy of $\cape$ Algorithm (Algorithm \ref{alg:dp_avg})]\label{thm:cape}
Consider Algorithm \ref{alg:dp_avg} in the distributed data setting of Section \ref{sec:problem_formulation} with $N_s = \frac{N}{S}$ and $\tau_s^2 = \tau^2$ for all sites $s \in [S]$. Suppose that at most $S_C = \ceil*{\frac{S}{3}} - 1$ sites can collude after execution. Then Algorithm \ref{alg:dp_avg} guarantees $(\epsilon, \delta)$-differential privacy for each site, where $(\epsilon,\delta)$ satisfy the relation $\delta = 2\frac{\sigma_z}{\epsilon - \mu_z}\phi\left(\frac{\epsilon - \mu_z}{\sigma_z}\right)$, $\phi(\cdot)$ is the density for standard Normal random variable and $(\mu_z, \sigma_z)$ are given by
\begin{align}
    \mu_z  &= \frac{S^3}{2\tau^2 N^2 (1+S)} \left(\frac{S - S_C + 2}{S - S_C} + \frac{\frac{9}{S - S_C}S_C^2}{S(1+S) - 3S_C^2}\right), 
    	\label{eq:CAPEmu} \\
    \sigma_z^2 &= \frac{S^3}{\tau^2 N^2 (1+S)} \left(\frac{S - S_C + 2}{S - S_C} + \frac{\frac{9}{S - S_C}S_C^2}{S(1+S) - 3S_C^2}\right).
    	\label{eq:CAPEsigma} 
\end{align}
\end{theorem}

\begin{Rem}
Theorem~\ref{thm:cape} is stated for the symmetric setting: $N_s = \frac{N}{S}$ and $\tau_s^2 = \tau^2\ \forall s \in [S]$. As with many algorithms using the approximate differential privacy, the guarantee holds for a range of $(\epsilon,\delta)$ pairs subject to a tradeoff constraint between $\epsilon$ and $\delta$, as in the simple case in \eqref{eq:GaussMech}.
\end{Rem}

\begin{proof}
As mentioned before, we identify the $S_H$ non-colluding sites with $s \in \{1, \ldots, S_H\} \triangleq \mathbb{S}_H$ and the $S_C$ colluding sites with $s \in \{S_H + 1, \ldots, S\} \triangleq \mathbb{S}_C$. The adversary can observe the outputs from each site (including the aggregator). Additionally, the colluding sites can share their private data and the noise terms, $\hat{e}_s$ and $g_s$ for $s \in \mathbb{S}_C$, with the adversary. For simplicity, we assume that all sites have equal number of samples (i.e., $N_s = \frac{N}{S}$) and $\tau_s^2 = \tau^2$.

To infer the private data of the sites $s \in \mathbb{S}_H$, the adversary can observe $\hat{\vect{a}} = \left[\hat{a}_1, \ldots, \hat{a}_{S_H}\right]^\top \in \mathbb{R}^{S_H}$ and $\hat{e} = \sum_{s \in \mathbb{S}_H} \hat{e}_s$. Note that the adversary can learn the partial sum $\hat{e}$ because they can get the sum $\sum_s \hat{e}_s$ from the aggregator and the noise terms $\{\hat{e}_{S_H + 1}, \ldots, \hat{e}_S\}$ from the colluding sites. Therefore, the vector $\vect{y} = \left[\hat{\vect{a}}^\top, \hat{e}\right]^\top \in \mathbb{R}^{S_H + 1}$ is what the adversary can observe to make inference about the non-colluding sites. To prove differential privacy guarantee, we must show that $\left|\log\frac{g(\vect{y} | \vect{a})}{g(\vect{y} | \vect{a}')}\right| \leq \epsilon$ holds with probability (over the randomness of the mechanism) at least $1-\delta$. Here, $\vect{a} = \left[f(\vect{x}_1), \ldots, f(\vect{x}_{S_H})\right]^\top$ and $g(\cdot | \vect{a})$ and $g(\cdot | \vect{a}')$ are the probability density functions of $\vect{y}$ under $\vect{a}$ and $\vect{a}'$, respectively. The vectors $\vect{a}$ and $\vect{a}'$ differ in only one coordinate (neighboring). Without loss of generality, we assume that $\vect{a}$ and $\vect{a}'$ differ in the first coordinate. We note that the maximum difference is $\frac{1}{N_s}$ as the sensitivity of the function $f(\vect{x_s})$ is $\frac{1}{N_s}$. Recall that we release $\hat{a}_s = f(\vect{x}_s) + e_s + g_s$ from each site. We observe $\forall s \in [S]$: $\mathbb{E}(\hat{a}_s) = f(\vect{x}_s),\ \mrm{var}(\hat{a}_s) = \tau^2$. Additionally, $\forall s_1 \neq s_2 \in [S]$, we have: $\mathbb{E}(\hat{a}_{s_1} \hat{a}_{s_2}) = f(\vect{x}_{s_1}) f(\vect{x}_{s_2}) - \frac{\tau^2}{S}$. That is, the random variable $\hat{\vect{a}}$ is $\mathcal{N}(\vect{a}, \Sigma_{\hat{\vect{a}}})$, where $\Sigma_{\hat{\vect{a}}} = (1+\frac{1}{S})\tau^2 \matr{I} - \vect{1}\vect{1}^\top \frac{\tau^2}{S} \in \mathbb{R}^{S_H \times S_H}$ and $\vect{1}$ is a vector of all ones. Without loss of generality, we can assume~\cite{dwork2013algorithmic} that $\vect{a} = \vect{0}$ and $\vect{a}' = \vect{a} - \vect{v}$, where $\vect{v} = \left[\frac{1}{N_s},0, \ldots, 0\right]^\top$. Additionally, the random variable $\hat{e}$ is $\mathcal{N}(0, \tau^2_{\hat{e}})$, where $\tau^2_{\hat{e}} = S_H \tau^2$. Therefore, $g(\vect{y} | \vect{a})$ is the density of $\mathcal{N}(\vect{0}, \Sigma)$, where
\[
\Sigma =
\begin{bmatrix}
    \Sigma_{\hat{\vect{a}}} 										& \Sigma_{\hat{\vect{a}}\hat{e}} \\
    \Sigma_{\hat{\vect{a}}\hat{e}}^\top			& \tau^2_{\hat{e}}
\end{bmatrix} \in \mathbb{R}^{(S_H + 1) \times (S_H + 1)}.
\]
With some simple algebra, we can find the expression for $\Sigma_{\hat{\vect{a}}\hat{e}}$: $\Sigma_{\hat{\vect{a}}\hat{e}} = \left(1-\frac{S_H}{S}\right)\tau^2 \vect{1} \in \mathbb{R}^{S_H}$. If we denote $\tilde{\vect{v}} = \left[\vect{v}^\top, 0\right]^\top \in \mathbb{R}^{S_H+1}$ then we observe
\begin{align*}
    \left|\log\frac{g(\vect{y} | \vect{a})}{g(\vect{y} | \vect{a}')}\right| &= \left|-\frac{1}{2}\left( \vect{y}^\top\Sigma^{-1}\vect{y} - \left(\vect{y} + \tilde{\vect{v}}\right)^\top\Sigma^{-1}\left(\vect{y} + \tilde{\vect{v}}\right)\right)\right| \\
    &= \left|\frac{1}{2}\left( 2\vect{y}^\top\Sigma^{-1}\tilde{\vect{v}} + \tilde{\vect{v}}^\top\Sigma^{-1}\tilde{\vect{v}}\right)\right| \\
    &= \left|\vect{y}^\top\Sigma^{-1}\tilde{\vect{v}} + \frac{1}{2}\tilde{\vect{v}}^\top\Sigma^{-1}\tilde{\vect{v}}\right| = |z|,
\end{align*}
where $z = \vect{y}^\top\Sigma^{-1}\tilde{\vect{v}} + \frac{1}{2}\tilde{\vect{v}}^\top\Sigma^{-1}\tilde{\vect{v}}$. Using the matrix inversion lemma for block matrices~\cite[Section 0.7.3]{horn2012} and some algebra, we have
\[
\Sigma^{-1} = 
\begin{bmatrix}
    \Sigma_{\hat{\vect{a}}}^{-1} + \frac{1}{K} \Sigma_{\hat{\vect{a}}}^{-1} \Sigma_{\hat{\vect{a}}\hat{e}}	 \Sigma_{\hat{\vect{a}}\hat{e}}^\top\Sigma_{\hat{\vect{a}}}^{-1}								& -\frac{1}{K} \Sigma_{\hat{\vect{a}}}^{-1} \Sigma_{\hat{\vect{a}}\hat{e}} \\
    -\frac{1}{K}\Sigma_{\hat{\vect{a}}\hat{e}}^\top\Sigma_{\hat{\vect{a}}}^{-1}			& \frac{1}{K}
\end{bmatrix},
\]
where $\Sigma_{\hat{\vect{a}}}^{-1} = \frac{S}{(1+S)\tau^2}\left(\matr{I} + \frac{2}{S_H}\vect{1}\vect{1}^\top\right)$ and $K = \tau^2_{\hat{e}} - \Sigma_{\hat{\vect{a}}\hat{e}}^\top\Sigma_{\hat{\vect{a}}}^{-1}\Sigma_{\hat{\vect{a}}\hat{e}}$. Note that $z$ is a Gaussian random variable $\mathcal{N}(\mu_z, \sigma_z^2)$ with parameters $\mu_z 	= \frac{1}{2}\tilde{\vect{v}}^\top\Sigma^{-1}\tilde{\vect{v}}$ and $\sigma_z^2 = \tilde{\vect{v}}^\top\Sigma^{-1}\tilde{\vect{v}}$ given by \eqref{eq:CAPEmu} and \eqref{eq:CAPEsigma}, respectively. Now, we observe
\begin{align*}
    \Pr\left[\left|\log\frac{g(\vect{y} | \vect{a})}{g(\vect{y} | \vect{a}')}\right| \leq \epsilon\right] &= \Pr\left[\left|z\right| \leq \epsilon\right] = 1 - 2 \Pr\left[z > \epsilon\right] \\
    &= 1 - 2Q\left(\frac{\epsilon - \mu_z}{\sigma_z}\right)\\
    &> 1 - 2\frac{\sigma_z}{\epsilon - \mu_z}\phi\left(\frac{\epsilon - \mu_z}{\sigma_z}\right),
\end{align*}
where $Q(\cdot)$ is the Q-function~\cite{qfunc} and $\phi(\cdot)$ is the density for standard Normal random variable. The last inequality follows from the bound $Q(x) < \frac{\phi(x)}{x}$~\cite{qfunc}. Therefore, the proposed $\cape$ ensures $(\epsilon, \delta)$-DP with $\delta = 2\frac{\sigma_z}{\epsilon - \mu_z}\phi\left(\frac{\epsilon - \mu_z}{\sigma_z}\right)$ for each site, assuming that the number of colluding sites is at-most $\ceil*{\frac{S}{3}} - 1$. As the local datasets are disjoint and differential privacy is invariant under post processing, the release of $a_\mrm{cape}$ also satisfies $(\epsilon, \delta)$-DP.
\end{proof}

\begin{Rem}
We use the $\secureagg$ protocol~\cite{Bonawitz17} to generate the zero-sum noise terms by mapping floating point numbers to a finite field. Such mappings are shown to be vulnerable to certain attacks~\cite{mironov2012}. However, the floating point implementation issues are out of scope for this paper. We refer the reader to the work of Balcer and Vadhan~\cite{balcer2018} for possible remedies. We believe a very interesting direction of future work would be to address the issue in our distributed data setting.
\end{Rem}

\subsection{Utility Analysis}\label{sec:cape_utility}
The goal is to ensure $(\epsilon, \delta)$-DP for each site and achieve $\tau_\mrm{cape}^2 = \tau_\mrm{pool}^2$ at the aggregator (see Lemma~\ref{lemma:cape}). The $\cape$ protocol guarantees $(\epsilon, \delta)$-DP with $\delta = 2\frac{\sigma_z}{\epsilon - \mu_z}\phi\left(\frac{\epsilon - \mu_z}{\sigma_z}\right)$. We claim that this $\delta$ guarantee is much better than the $\delta$ guarantee in the conventional distributed DP scheme. We empirically validate this claim by comparing $\delta$ with $\delta_\mrm{conv}$ in Appendix~\ref{appendix:eff_delta} in the Supplement. Here, $\delta_\mrm{conv}$ is the smallest $\delta$ guarantee we can afford in the conventional distributed DP scheme to achieve the same noise variance as the pooled-data scenario for a given $\epsilon$. Additionally, we empirically compare $\delta$ and $\delta_\mrm{conv}$ for weaker collusion assumptions in Appendix~\ref{appendix:fewer_colluding_sites} in the Supplement. In both cases, we observe that $\delta$ is always smaller than $\delta_\mrm{conv}$. That is, for achieving the same noise level at the aggregator output (and therefore the same utility) as the pooled data scenario, we are ensuring a much better privacy guarantee by employing the $\cape$ scheme over the conventional approach. 

\begin{lemma}\label{lemma:cape}
Consider the symmetric setting: $N_s = \frac{N}{S}$ and $\tau_s^2 = \tau^2$ for all sites $s \in [S]$. Let the variances of the noise terms $e_s$ and $g_s$ (Step \ref{alg:dp_avg:step_as_hat} of Algorithm \ref{alg:dp_avg}) be $\tau_e^2 = \left(1-\frac{1}{S}\right)\tau^2$ and $\tau_g^2 = \frac{\tau^2}{S}$, respectively. If we denote the variance of the additive noise (for preserving privacy) in the pooled data scenario by $\tau_\mrm{pool}^2$ and the variance of the estimator $a_\mrm{cape}$ (Step \ref{alg:dp_avg:step_a_ag} of Algorithm \ref{alg:dp_avg}) by $\tau_\mrm{cape}^2$ then Algorithm \ref{alg:dp_avg} achieves the same expected error as the pooled-data scenario (i.e., $\tau_\mrm{pool}^2 = \tau_\mrm{cape}^2$). 
\end{lemma}

\begin{proof}
The proof is given in Appendix \ref{appendix:cape_lemma} in the Supplement.
\end{proof}

\begin{prop}\label{prop:gain}(Performance gain) If the local noise variances are $\{\tau_s^2\}$ for $s \in [S]$ then the $\cape$ algorithm achieves a gain of $G = \frac{\tau_\mrm{conv}^2}{\tau_\mrm{cape}^2} = S$, where  $\tau_\mrm{conv}^2$ and $\tau_\mrm{cape}^2$ are the noise variances of the final estimate at the aggregator in the conventional distributed DP scheme and the $\cape$ scheme, respectively.
\end{prop}

\begin{proof}
The proof is given in Appendix \ref{appendix:perf_gain} in the Supplement.
\end{proof}
Note that, even in the case of site drop-out, we achieve $\sum_s e_s = 0$, as long as the number of active sites is above some threshold (see Bonawitz et al.~\cite{Bonawitz17} for details). Therefore, the performance gain of $\cape$ remains the same irrespective of the number of dropped-out sites.

\begin{Rem}[Unequal Sample Sizes at Sites]
Note that the $\cape$ algorithm achieves the same noise variance as the pooled-data scenario (i.e., $\tau_\mrm{cape}^2 = \tau_\mrm{pool}^2$) in the symmetric setting: $N_s = \frac{N}{S}$ and $\tau_s^2 = \tau^2$ for all sites $s \in [S]$. In general, the ratio $H(\vect{n}) = \frac{\tau_\mrm{cape}^2}{\tau_\mrm{pool}^2}$, where $\vect{n} \triangleq [N_1,\ N_2,\ \ldots, N_S]$, is a function of the sample sizes in the sites. We observe: $H(\vect{n}) = \frac{N^2}{S^3} \sum_{s=1}^S \frac{1}{N_s^2}$. As $H(\vect{n})$ is a Schur-convex function, it can be shown using majorization theory~\cite{majorization} that $1 \leq H(\vect{n}) \leq \frac{N^2}{S^3}\left(\frac{1}{\left(N - S + 1\right)^2} + S - 1\right)$, where the minimum is achieved for the symmetric setting. That is, $\cape$ achieves the smallest noise variance at the aggregator in the symmetric setting. For distributed systems with unequal sample sizes at the sites and/or different $N_s$ and $\tau_s$ at each site, we compute the weighted sum $a_\mrm{cape} = \sum_{s=1}^S w_s \hat{a}_s$ at the aggregator. In order to achieve the same noise level as the pooled data scenario, we need to ensure $\sum_{s=1}^S w_s e_s = 0$ and $\text{var}\left[\sum_{s=1}^s w_s g_s\right] = \tau_\mrm{pool}^2$. A scheme for doing so is shown in~\cite{imtiaz2018}. In this paper, we keep the analysis for the case $N_s = \frac{N}{S}, \forall s\in [S]$ for simplicity. 
\end{Rem}

\subsection{Scope of $\cape$}\label{sec:cape_scope}
$\cape$ is motivated by scientific research collaborations that are common in medicine and biology. Privacy regulations prevent sites from sharing the local raw data. Additionally, the data is often high dimensional (e.g., in neuroimaging) and sites have small sample sizes. Joint learning across datasets can yield discoveries that are impossible to obtain from a single site. $\cape$ can benefit functions $f$ with sensitivities satisfying some conditions (see Proposition \ref{prop:low_sensitivity}). In addition to the averaging function, many functions of interest have sensitivities that satisfy such conditions. Examples include the empirical average loss functions used in machine learning and deep neural networks. Additionally, we can use the Stone-Weierstrass theorem~\cite{rudin1976} to approximate a loss function $f$ and apply $\cape$, as we show in Section \ref{sec:capefm}. Furthermore, we can use the nomographic representation of functions to approximate a desired function in a decentralized manner~\cite{Kolmogoro57:rep1,Sprecher65representation,buck1982nomographic,Sprecher14representation} (for applications in communications~\cite{Nazer07,LimmerMS:15nomographic,Goldenbaum13tcom,Goldenbaum13tsp}), while keeping the data differentially private. More common applications include gradient based optimization algorithms, $k$-means clustering and estimating probability distributions.

\begin{prop}\label{prop:low_sensitivity}
Consider a distributed setting with $S > 1$ sites in which site $s \in [S]$ has a dataset $\mathbb{D}_s$ of $N_s$ samples and $\sum_{s=1}^S N_s = N$. Suppose the sites are computing a function $f(\mathbb{D})$ with $\mathcal{L}_p$ sensitivity $\Delta(N)$ employing the $\cape$ scheme. Denote $\vect{n} = [N_1,\ N_2,\ \ldots, N_S]$ and observe the ratio $H(\vect{n}) = \frac{\tau_\mrm{cape}^2}{\tau_\mrm{pool}^2} = \frac{\sum_{s=1}^S \Delta^2(N_s)}{S^3 \Delta^2(N)}$. Then the $\cape$ protocol achieves $H(\vect{n}) = 1$, if 
\begin{itemize}
\item for convex $\Delta (N)$ we have: $\Delta\left(\frac{N}{S}\right) = S \Delta (N)$
\item for general $\Delta (N)$ we have: $S^3 \Delta^2(N) = \sum_{s=1}^S \Delta^2(N_s)$.
\end{itemize}
\end{prop}

\begin{proof}
We review some definitions and lemmas~\cite[Proposition C.2]{majorization} necessary for the proof in Appendix~\ref{appendix:low_sensitivity} in the Supplement. As the sites are computing the function $f$ with $\mathcal{L}_p$ sensitivity $\Delta(N)$, the local noise standard deviation for preserving privacy is proportional to $\Delta(N_s)$ by Gaussian mechanism~\cite{dwork2006}. It can be written as: $\tau_s = \Delta(N_s) C$, where $C$ is a constant for a given $(\epsilon, \delta)$ pair. Similarly, the noise standard deviation in the pooled data scenario can be written as: $\tau_\mrm{pool} = \Delta(N) C$. Now, the final noise variance at the aggregator for $\cape$ protocol is: $\tau^2_\mrm{cape} = \sum_{s=1}^S \frac{\tau_g^2}{S^2} = \frac{1}{S^3} \sum_{s=1}^S \Delta^2(N_s) C^2$. Now, we observe the ratio: $H(\vect{n}) = \frac{\tau_\mrm{cape}^2}{\tau_\mrm{pool}^2} = \frac{\sum_{s=1}^S \Delta^2(N_s)}{S^3 \Delta^2(N)}$. As we want to achieve the same noise variance as the pooled-data scenario, we need $S^3 \Delta^2(N) = \sum_{s=1}^S \Delta^2(N_s)$, which proves the case for general sensitivity function $\Delta(N)$. Now, if $\Delta^2(N)$ is convex then the by Lemma~\ref{lem: schurconx} (Supplement) the function $K(\vect{n}) = \sum_{s=1}^S \Delta^2(N_s)$ is Schur-convex. Thus the minimum of $K(\vect{n})$ is obtained when $\vect{n} = \vect{n}_\mrm{sym}$. We observe: $K_{\min}(\vect{n}) = \sum_{s=1}^S \Delta^2\left(\frac{N}{S}\right) = S \cdot 	\Delta^2\left(\frac{N}{S}\right)$. Therefore, when $\Delta(N)$ is convex, we achieve $H(\vect{n}) = 1$ if $\Delta(\frac{N}{S}) = S \Delta(N)$. 
\end{proof}

\section{Application of $\cape$: Distributed Computation of Functions}
As mentioned before, the $\cape$ framework can benefit any distributed differentially private function computation, as long as the sensitivity of the function satisfies the conditions outlined in Proposition \ref{prop:low_sensitivity}. In this section, we propose an algorithm that is specifically suited for privacy-preserving computation of cost functions in distributed settings. Let us consider a cost function $f_D(\vect{w}): \mathbb{R}^D \mapsto \mathbb{R}$ that depends on private data distributed across $S$ sites. A central aggregator (see Figure \ref{fig:network_structure}) wishes find the minimizer $\vect{w}^*$ of $f_D(\vect{w})$. This is a common scenario in distributed machine learning. Now, the aggregator is not trusted and the sites may collude with an adversary to learn information about the other sites. Since computing the $\vect{w}^*$ by minimizing the expected cost/loss involves the sensitive information of the local datasets, we need to ensure that $\vect{w}^*$ is computed in a privacy-preserving way. In particular, we want to develop an algorithm to compute the $(\epsilon, \delta)$ differentially private approximate to $\vect{w}^*$, denoted $\hat{\vect{w}}^*$, in a distributed setting that produces a result as close as possible to the non-private pooled $\vect{w}^*$.

Inline with our discussions in the previous sections, let us assume that each site $s \in [S]$ holds a dataset $\mathbb{D}_s$ of $N_s$ samples $\vect{x}_{s, n} \in \mathbb{R}^D$. The total sample size across all sites is $N = \sum_{s=1}^S N_s$. The cost incurred by a $\vect{w} \in \mathbb{R}^D$ due to one data sample $\vect{x}_{s, n}$ is $f(\vect{x}_{s, n}; \vect{w}) : \mathbb{R}^D \times \mathbb{R}^D \mapsto \mathbb{R}$. We need to minimize the average cost to find the optimal $\vect{w}^*$. The empirical average cost for a particular $\vect{w}$ over all the samples is expressed as
\begin{align}\label{eqn:empirical_avg_cost}
f_D(\vect{w}) &= \frac{1}{N} \sum_{s=1}^S \sum_{n=1}^{N_s} f(\vect{x}_{s,n}; \vect{w}).
\end{align}
Therefore, we have
\begin{align*}
\vect{w}^* &= \argmin_\vect{w} f_D(\vect{w}) = \argmin_\vect{w} \frac{1}{N} \sum_{s=1}^S \sum_{n=1}^{N_s} f(\vect{x}_{s,n}; \vect{w}).
\end{align*}
For centralized optimization, we can find a differentially-private approximate to  $\vect{w}^*$ using output perturbation~\cite{anand2011} or objective perturbation~\cite{anand2011,nozari2016}. We propose using the \emph{functional mechanism}~\cite{zhang2012}, which is a form of objective perturbation and is more amenable to distributed implementation. It uses a polynomial representation of the cost function and can be used for any differentiable and continuous cost function. We perturb each term in the polynomial representation of $f_D(\vect{w})$ to get modified cost function $\hat{f}_D(\vect{w})$. The minimizer $\hat{\vect{w}}^* = \argmin_\vect{w} \hat{f}_D(\vect{w})$ guarantees differential privacy.

\subsection{Functional Mechanism}
We first review the functional mechanism~\cite{zhang2012} in the pooled data case. Suppose $\phi(\vect{w})$ is a monomial function of the entries of $\vect{w}$: $\phi(\vect{w}) = w_1^{c_1} w_2^{c_2} \ldots w_D^{c_D}$, for some set of exponents $\{c_d : d \in [D]\} \subseteq \mathbb{N}$. Let us define the set $\Phi_j$ of all $\phi(\vect{w})$ with degree $j$ as
\begin{align*}
\Phi_j &= \left\{w_1^{c_1} w_2^{c_2} \ldots w_D^{c_D} : \sum_{d=1}^D c_d = j \right\}.
\end{align*}
For example, $\Phi_0 = \{1\}$, $\Phi_1 = \{w_1,\ w_2,\ \ldots,\ w_D\}$, $\Phi_2 = \{w_{d_1} w_{d_2} : d_1, d_2 \in [D]\}$, etc. Now, from the Stone-Weierstrass Theorem~\cite{rudin1976}, any differentiable and continuous cost function $f(\vect{x}_n; \vect{w})$ can be written as a (potentially infinite) sum of monomials of $\{w_d\}$:
\begin{align*}
f(\vect{x}_n; \vect{w}) &= \sum_{j=0}^J \sum_{\phi \in \Phi_j}\lambda_{n\phi}\phi(\vect{w}),
\end{align*}
for some $J \in [0,\infty]$. Here, $\lambda_{n\phi}$ denotes the coefficient of $\phi(\vect{w})$ in the polynomial and is a function of the $n$-th data sample. Plugging the expression of $f(\vect{x}_n; \vect{w})$ in \eqref{eqn:empirical_avg_cost}, we can express the empirical average cost over all $N$ samples as
\begin{align}\label{eqn:empirical_cost}
f_D(\vect{w}) &= \sum_{j=0}^J \sum_{\phi \in \Phi_j}\left(\frac{1}{N} \sum_{n=1}^N \lambda_{n\phi}\right)\phi(\vect{w}).
\end{align}
The function $f_D(\vect{w})$ depends on the data samples only through $\{\lambda_{n\phi}\}$. As the goal is to approximate $f_D(\vect{w})$ in a differentially-private way, one can compute these $\lambda_{n\phi}$ to satisfy differential privacy~\cite{zhang2012}. Let us consider two ``neighboring'' datasets: $\mathbb{D} = \{\vect{x}_1,\ \ldots,\ \vect{x}_{N-1},\ \vect{x}_N\}$ and $\mathbb{D}' = \{\vect{x}_1,\ \ldots,\ \vect{x}_{N-1},\ \vect{x'}_N\}$, differing in a single data point (i.e., the last one). Zhang et al.~\cite{zhang2012} computed the $\mathcal{L}_1$ sensitivity $\Delta^\mrm{dp-fm}$ as
\begin{align*}
\left\|\sum_{j=0}^J \sum_{\phi \in \Phi_j} \frac{1}{N} \left(\sum_{\mathbb{D}} \lambda_{n\phi} - \sum_{\mathbb{D}'} \lambda_{n\phi}\right)\right\|_1 \\
\leq \frac{2}{N} \max_n \sum_{j=0}^J \sum_{\phi \in \Phi_j} \|\lambda_{n\phi}\|_1 \triangleq \Delta^\mrm{dp-fm},
\end{align*}
and proposed an $\epsilon$-differentially private method that adds Laplace noise with variance $2\left(\frac{\Delta^\mrm{dp-fm}}{\epsilon}\right)^2$ to each $\sum_\mathbb{D} \lambda_{n\phi}$ for all $\phi \in \Phi_j$ and for all $j \in \{0,\ldots,J\}$. In the following, we propose an improved functional mechanism that employs a tighter characterization of the sensitivities and can be extended to incorporate the $\cape$ protocol for the distributed settings.

\subsection{Improved Functional Mechanism}\label{sec:improved_fm}
\begin{algorithm}[t] 
	\caption{Improved Functional Mechanism \label{alg:fm}}
	\begin{algorithmic}[1]
    \Require Data samples $\{\vect{x}_{n}\}$; cost function $f_D(\vect{w})$ as in \eqref{eqn:f_D_with_Lambda_j}; privacy parameters $\epsilon$, $\delta$
    \For{$j = 0,\ 1,\ \ldots,\ J$}
    	\State Compute $\Lambda_j$ as shown in Section \ref{sec:improved_fm}
    	\State Compute $\Delta_j \gets \max_{\mathbb{D}, \mathbb{D}'} \|\Lambda_j^\mathbb{D} - \Lambda_j^{\mathbb{D}'}\|_F$
        \State Compute $\tau_j = \frac{\Delta_j}{\epsilon}\sqrt{2\log\frac{1.25}{\delta}}$
    	\State Compute $\hat{\Lambda}_j \gets \Lambda_j + e_j$, where $e_j$ have same dimensions as $\Lambda_j$ and entries of $e_j$ are i.i.d. $\sim \mathcal{N}(0, \tau_j^2)$
    \EndFor
    \State Compute $\hat{f}_D(\vect{w}) = \sum_{j=0}^J \inprod{\hat{\Lambda}_j}{\bar{\phi}_j}$\\
    \Return $\hat{f}_D(\vect{w})$
    \end{algorithmic}
\end{algorithm}
Our method is an improved version of the functional mechanism~\cite{zhang2012}. We use the Gaussian mechanism~\cite{dwork2013algorithmic} for computing the $(\epsilon, \delta)$-DP approximate of $f_D(\vect{w})$. This gives a weaker privacy guarantee than the original functional mechanism~\cite{zhang2012}, which used Laplace noise for $\epsilon$-DP. Our distributed function computation method (described in Section \ref{sec:capefm}) benefits from the fact that linear combinations of Gaussians are Gaussian. In other words, the proposed $\cape$ and the distributed functional mechanism rely on the Gaussianity of the noise. Now, instead of computing the $\mathcal{L}_2$-sensitivity of $\lambda_{n\phi}$, we define an \emph{array} $\Lambda_j$ that contains $\frac{1}{N}\sum_{n=1}^N \lambda_{n\phi}$ as its entries for all $\phi(\vect{w}) \in \Phi_j$. We used the term ``array'' instead of vector or matrix or tensor because the dimension of $\Lambda_j$ depends on the cardinality $|\Phi_j|$ of the set $\Phi_j$. 
We can represent $\Lambda_j$ as a scalar for $j = 0$ (because $\Phi_0 = \{1\}$), as a $D$-dimensional vector for $j = 1$ (because $\Phi_1 = \{w_d : d \in [D]\}$), and as a $D\times D$ matrix for $j = 2$ (because $\Phi_2 = \{w_{d_1} w_{d_2} : d_1, d_2 \in [D]\}$). Additionally, we use $\bar{\phi}_j$ to denote the array containing all $\phi(\vect{w}) \in \Phi_j$ as its entries. The arrays $\bar{\phi}_j$ and $\Lambda_j$ have the same dimensions and number of elements. Rewriting the objective, we observe
\begin{align}
f_D(\vect{w}) 	&= \sum_{j=0}^J \sum_{\phi \in \Phi_j}\left(\frac{1}{N} \sum_{n=1}^N \lambda_{n\phi}\right)\phi(\vect{w})= \sum_{j=0}^J \inprod{\Lambda_j}{\bar{\phi}_j}\label{eqn:f_D_with_Lambda_j}.
\end{align}
We define the $\mathcal{L}_2$-sensitivity of $\Lambda_j$ as
\begin{align}\label{eqn:sensitivity_Lambda}
\Delta_j &= \max_{\mathbb{D}, \mathbb{D}'} \|\Lambda_j^\mathbb{D} - \Lambda_j^{\mathbb{D}'}\|_F,
\end{align}
where $\Lambda_j^\mathbb{D}$ and $\Lambda_j^{\mathbb{D}'}$ are computed on neighboring datasets $\mathbb{D}$ and $\mathbb{D}'$, respectively. Now, we use the Gaussian mechanism to get an  $(\epsilon, \delta)$-DP approximation to $\Lambda_j$ as: $\hat{\Lambda}_j = \Lambda_j + e_j$, where $e_j$ is a noise array with the same dimensions as $\Lambda_j$ and its entries are drawn i.i.d. $\sim \mathcal{N}(0, \tau_j^2)$. Here $\tau_j = \frac{\Delta_j}{\epsilon}\sqrt{2\log\frac{1.25}{\delta}}$. As the function $f_D(\vect{w})$ depends on the data only through $\{\Lambda_j\}$, the computation
\begin{align}\label{eqn:hat_f_D_with_hat_Lambda_j}
\hat{f}_D(\vect{w}) &= \sum_{j=0}^J \inprod{\hat{\Lambda}_j}{\bar{\phi}_j}
\end{align}
satisfies $(\epsilon, \delta)$ differential privacy. We show the proposed improved functional mechanism in Algorithm \ref{alg:fm}.

\begin{theorem}\label{thm:fm}
Consider Algorithm \ref{alg:fm} with input parameters $(\epsilon, \delta)$ and the function $f_D(\vect{w})$ represented as \eqref{eqn:f_D_with_Lambda_j}. Then Algorithm \ref{alg:fm} computes an $(\epsilon, \delta)$ differentially private approximation $\hat{f}_D(\vect{w})$ to $f_D(\vect{w})$. Consequently, the minimizer $\hat{\vect{w}}^* = \argmin_\vect{w} \hat{f}_D(\vect{w})$ satisfies $(\epsilon, \delta)$-differential privacy.
\end{theorem}

\begin{proof} The proof is given in Appendix~\ref{appendix:thm:fm} in the Supplement.
\end{proof}

\subsection{Example -- Linear Regression}\label{sec:fm_example}
In this section, we demonstrate the proposed improved functional mechanism in the centralized setting using a linear regression problem. We show that Algorithm \ref{alg:fm} achieves much better utility at the price of a weaker privacy guarantee ($(\epsilon, \delta)$-DP vs $\epsilon$-DP). The main reason for this performance improvement is due to defining the sensitivities of $\Lambda_j$ separately for each $j$, instead of using an uniform conservative upper-bound $\Delta^\mrm{dp-fm}$ (as in~\cite{zhang2012}). We demonstrate the proposed improved functional mechanism for a logistic regression problem in Appendix~\ref{appendix:log_reg} in the Supplement.

We have a dataset $\mathbb{D}$ with $N$ samples of the form $(\vect{x}_n, y_n)$, where $\vect{x}_n \in \mathbb{R}^D$ is the feature vector and $y_n \in \mathbb{R}$ is the response. Without loss of generality, we assume that $\|\vect{x}_n\|_2 \leq 1$ and $y_n \in [-1, 1]$. We want to find a vector $\vect{w} \in \mathbb{R}^D$ such that $\vect{x}_n^\top\vect{w} \approx y_n$ for all $n\in [N]$. The cost function $f:\mathbb{R}^D \times \mathbb{R}^D \mapsto \mathbb{R}^+$ due to each sample and a particular $\vect{w}$ is the squared loss: $f(\vect{x}_n, \vect{w}) = \left(y_n - \vect{x}_n^\top\vect{w}\right)^2$. The empirical average cost function is defined as
\begin{align}\label{eqn:loss_avg}
f_D(\vect{w}) 	&= \frac{1}{N} \sum_{n=1}^N f(\vect{x}_n, \vect{w}) = \frac{1}{N} \left\|\vect{y} - \matr{X}\vect{w}\right\|_F^2,
\end{align}
where $\vect{y} \in \mathbb{R}^N$ contains all $y_n$ and $\matr{X} \in \mathbb{R}^{N\times D}$ contains $\vect{x}_n^\top$ as its rows. By simple algebra, it can be shown that
\begin{align*}
f_D(\vect{w}) &= \left(\frac{1}{N} \sum_{n=1}^N y_n^2\right) + \sum_{d=1}^D\left(-\frac{2}{N}\sum_{n=1}^N y_n x_{nd}\right)w_d \\
&\qquad + \sum_{d_1 = 1}^D\sum_{d_2 = 1}^D \left(\frac{1}{N}\sum_{n=1}^N x_{nd_1} x_{nd_2}\right)w_{d_1}w_{d_2},
\end{align*}
where $x_{pq}$ is the $q$-th element of the $p$-th sample vector or $x_{pq} = [\matr{X}]_{pq}$. Now, we have
\begin{align*}
\Lambda_0 &= \frac{1}{N} \sum_{n=1}^N y_n^2 = \frac{1}{N}\|\vect{y}\|_2^2\\
\Lambda_1 &= -\frac{2}{N}	\begin{bmatrix}
    							\sum_{n=1}^N y_n x_{n1} \\
                				\vdots \\
    							\sum_{n=1}^N y_n x_{nD}
  							\end{bmatrix} = -\frac{2}{N} \left(\vect{1}^\top \mrm{diag}(\vect{y}) \matr{X}\right)^\top\\
\Lambda_2 &= \frac{1}{N} 	\begin{bmatrix}
										\sum_{n=1}^N x_{n1}^2 & \cdots & \sum_{n=1}^N x_{n1}x_{nD}\\
                                			\vdots					& \ddots & \vdots \\
                                			\sum_{n=1}^N x_{nD}x_{n1} & \cdots & \sum_{n=1}^N x_{nD}^2
										\end{bmatrix} = \frac{\matr{X}^\top\matr{X}}{N},
\end{align*}
where $\vect{1}$ is an $N$-dimensional vector of all 1's and $\mrm{diag}(\vect{y})$ is an $N\times N$ diagonal matrix with the elements of $\vect{y}$ as diagonal entries. Additionally, we write out the $\{\bar{\phi}_j\}$ for completeness
\begin{align*}
\bar{\phi}_0 &= 1,\ 
\bar{\phi}_1 = 	\begin{bmatrix}
			w_1 \\
			\vdots \\
			w_D
			\end{bmatrix},\ \mbox{and} \\
\bar{\phi}_2 &= 	\begin{bmatrix}
				w_1^2 & w_1w_2 & \cdots & w_1w_D \\
                \vdots & \vdots & \ddots & \vdots \\
                w_Dw_1 & w_Dw_2 & \cdots & w_D^2
			\end{bmatrix}.
\end{align*}
Therefore, we can express $f_D(\vect{w})$ as $f_D(\vect{w}) = \sum_{j=0}^2 \inprod{\Lambda_j}{\bar{\phi}_j}$. Now, we focus on finding the sensitivities of $\{\Lambda_j\}$. Let us consider a neighboring dataset $\mathbb{D}'$ which contains the same tuples as $\mathbb{D}$, except for the last one, i.e. $(\vect{x}_N', y_N')$. We have
\begin{align*}
\left|\Lambda_0^\mathbb{D} - \Lambda_0^{\mathbb{D}'}\right| &= \frac{1}{N} \left(y_N^2 - {y_N'}^2\right) \leq \frac{1}{N} \triangleq \Delta_0,
\end{align*}
where the inequality follows from the fact that $y_n \in [-1, 1]$. Next, we observe
\begin{align*}
\left\|\Lambda_1^\mathbb{D} - \Lambda_1^{\mathbb{D}'}\right\|_2 &= \frac{2}{N} \|\vect{a} - \vect{a}'\|_2 \leq  \frac{4}{N} \triangleq \Delta_1,
\end{align*}
where $\vect{a} = y_N\vect{x}_N$ and $\vect{a}' = y_N' \vect{x}_N'$. Here, we used the inequality $\|\vect{a}\|_2 \leq 1$ as $-1 \leq y_n \leq 1$, $\|\vect{x}_n\|_2 \leq 1$. Similarly, $\|\vect{a}'\|_2 \leq 1$. Finally, we observe
\begin{align*}
\left\|\Lambda_2^\mathbb{D} - \Lambda_2^{\mathbb{D}'}\right\|_F &= \frac{1}{N}\left\|\matr{X}^\top\matr{X} - {\matr{X}'}^\top\matr{X}'\right\|_F \\
							&= \frac{1}{N}\left\|\vect{x}_N\vect{x}_N^\top - \vect{x}_N'{\vect{x}_N'}^\top\right\|_F\\
                            &\leq \frac{\sqrt{2}}{N} \triangleq \Delta_2,
\end{align*}
where the last inequality follows from realizing that the $D\times D$ symmetric matrix $\vect{x}_N\vect{x}_N^\top - \vect{x}_N'{\vect{x}_N'}^\top$ is at most rank-2. Therefore, we can write its eigen-decomposition as:
\begin{align*}
\left\|\vect{x}_N\vect{x}_N^\top - \vect{x}_N'{\vect{x}_N'}^\top\right\|_F &= \left\|\matr{U}\matr{\Sigma}\matr{U}^\top\right\|_F \\
					&\leq \|\matr{U}\|_F^2\|\matr{\Sigma}\|_F = \|\matr{\Sigma}\|_F \leq \sqrt{2}.
\end{align*}
Now that we have computed the $\mathcal{L}_2$-sensitivities of $\{\Lambda_j\}$, we can compute $\{\hat{\Lambda}_j\}$ and thus, the $(\epsilon, \delta)$-DP approximation $\hat{f}_D(\vect{w})$ following Algorithm \ref{alg:fm}. Note that if we employed the sensitivity computation technique as in~\cite{zhang2012}, the sensitivity for each entry of $\Lambda_j$ would be $\Delta^\mrm{dp-fm} = \frac{2(D+1)^2}{N}$, which is orders of magnitude larger than $\Delta_j$ for any meaningful $D > 1$ and for all $j \in \{0,\ldots,J\}$. Therefore, with the sensitivity computation proposed in this work, we can achieve $\hat{f}_D(\vect{w})$ with much less noise. This results in a more accurate privacy-preserving estimate of the optimal $\hat{\vect{w}}^* = \argmin_\vect{w}\hat{f}_D(\vect{w})$, which we demonstrate empirically in Section \ref{sec:experimental_results}.

\subsection{Distributed Functional Mechanism}\label{sec:convdistfm}
In this section, we first we describe a conventional approach. Inline with Section \ref{sec:dist_avg}, we show that the conventional approach would always result in sub-optimal performance due to larger noise added to the cost function. Then we present the proposed algorithm ($\capefm$) in detail (see Algorithm \ref{alg:capefm}).

\noindent\textbf{Conventional Approach. }Recall the distributed setting, where we have $S$ different sites and a central aggregator, as depicted in Figure \ref{fig:network_structure}. The aggregator is not trusted. The goal is to compute the differentially private minimizer $\hat{\vect{w}}^*$ in this distributed setting. We assume that the datasets in the sites ($\mathbb{D}_s$ for $s \in [S]$) are disjoint and $N_s = \frac{N}{S}$. Recall that the cost function $f_D(\vect{w})$ depends on the data only through the $\{\Lambda_j\}$ for $j \in \{0,\ldots,J\}$. One na\"ive way to compute $\hat{\vect{w}}^*$ at the aggregator is the following: the sites use local data to compute and send the $(\epsilon, \delta)$ differentially-private approximates of $\{\Lambda_j^s\}$, denoted $\{\hat{\Lambda}_j^s\}$, to the aggregator. The aggregator combines these quantities to compute $\{\hat{\Lambda}_j\}$ and hence $\hat{f}_D(\vect{w})$. The aggregator can then compute and release $\hat{\vect{w}}^* = \argmin_\vect{w} \hat{f}_D(\vect{w})$. More specifically, each site $s$ computes $\hat{\Lambda}_j^s = \Lambda_j^s + e_j^s$, where $e_j^s$ is an array with the same dimensions as $\Lambda_j^s$ and the entries are drawn i.i.d. $\sim \mathcal{N}(0, {\tau_j^s}^2)$. Here,
\begin{align*}
\tau_j^s &= \frac{\Delta_j^s}{\epsilon}\sqrt{2\log\frac{1.25}{\delta}},
\end{align*}
and $\Delta_j^s = \max_{\mathbb{D}_s, \mathbb{D}_s'} \|\Lambda_j^{s,\mathbb{D}_s} - \Lambda_j^{s,\mathbb{D}_s'}\|_F$, where $\Lambda_j^{s,\mathbb{D}_s}$ and $\Lambda_j^{s,\mathbb{D}'_s}$ are computed on neighboring datasets $\mathbb{D}_s$ and $\mathbb{D}_s'$, respectively. Recall that $\Lambda_j^{s,\mathbb{D}_s}$ (and similarly $\Lambda_j^{s,\mathbb{D}'_s}$) is an array of $|\Phi_j|$ elements containing $\frac{1}{N_s}\sum_{n=1}^{N_s}\lambda_{n\phi}$ as its entries for all $\phi(\vect{w}) \in \Phi_j$. The sites send these $\{\hat{\Lambda}_j^s\}$ to the aggregator. The aggregator then computes
\begin{align*}
\hat{\Lambda}_j^\mrm{conv} &= \frac{1}{S} \sum_{s=1}^S \hat{\Lambda}_j^s = \frac{1}{S} \sum_{s=1}^S \Lambda_j^s + \frac{1}{S} \sum_{s=1}^S e_j^s.
\end{align*}
The variance of the estimator $\hat{\Lambda}_j^\mrm{conv}$ is $S\cdot \frac{{\tau_j^s}^2}{S^2} = \frac{{\tau_j^s}^2}{S} \triangleq {\tau_j^\mrm{conv}}^2$. However, if we had all the data samples at the aggregator (centralized/pooled data scenario), we could compute the $(\epsilon, \delta)$-DP approximates of $\{\Lambda_j\}$ as $\hat{\Lambda}_j = \Lambda_j + e_j$, where $e_j$ is an array with the same dimensions as $\Lambda_j$ and the entries are drawn i.i.d. $\sim \mathcal{N}(0, {\tau_j^\mrm{pool}}^2)$. The noise standard deviation $\tau_j^\mrm{pool}$ is given by:
\begin{align*}
\tau_j^\mrm{pool} &= \frac{\Delta_j^\mrm{pool}}{\epsilon}\sqrt{2\log\frac{1.25}{\delta}},
\end{align*}
where $\Delta_j^\mrm{pool} = \max_{\mathbb{D}, \mathbb{D}'} \|\Lambda_j^\mathbb{D} - \Lambda_j^{\mathbb{D}'}\|_F$. Now, $\Lambda_j^\mathbb{D}$ is an $|\Phi_j|$-dimensional array with entries $\frac{1}{N}\sum_{n=1}^{N}\lambda_{n\phi}$ for all $\phi(\vect{w}) \in \Phi_j$. Clearly $\frac{\Delta_j^\mrm{pool}}{\Delta_j^s} = \frac{1/N}{1/N_s} = \frac{1}{S}$, which implies $\tau_j^\mrm{pool} = \frac{\tau_j^s}{S}$. For this case, we observe the ratio $\frac{{\tau_j^\mrm{pool}}^2}{{\tau_j^\mrm{conv}}^2} = \frac{{\tau_j^s}^2 / S^2}{{\tau_j^s}^2 / S} = \frac{1}{S}$, which is exactly the same as the distributed problem of Section \ref{sec:dist_avg}. 

\subsection{Proposed Distributed Functional Mechanism $(\capefm)$}\label{sec:capefm}
\begin{algorithm}[t] 
	\caption{Proposed Distributed Functional Mechanism ($\capefm$)\label{alg:capefm}}
	\begin{algorithmic}[1]
    \Require Data samples $\{\vect{x}_{s,n}\}$; cost function $f_D(\vect{w})$ as in \eqref{eqn:f_D_with_Lambda_j}; local noise variances $\{\tau_j^s\}$ for all $j \in \{0,\ldots,J\}$
    \For{$s = 1,\ \ldots,\ S$}
    	\For{$j = 0,\ 1,\ \ldots,\ J$}
        	\State Compute $\Lambda_j$ as described in Section \ref{sec:improved_fm}
        	\State Generate $e_j^s$ according to Algorithm \ref{alg:zero-sum-noise-generation} (entrywise)
        	\State Compute ${\tau_{jg}^s}^2 \gets \frac{{\tau_j^s}^2}{S}$
         \State Generate $g_j^s$ with entries i.i.d. $\sim \mathcal{N}(0, {\tau_{jg}^s}^2)$
    		\State Compute $\hat{\Lambda}_j^s \gets \Lambda_j^s + e_j^s + g_j^s$ \label{alg:capefm:step_lambda_hat}
    	\EndFor
    \EndFor
    \State At the central aggregator, compute for all $j \in \{0,\ldots,J\}$: $\hat{\Lambda}_j \gets \frac{1}{S} \sum_{s=1}^S \hat{\Lambda}_j^s$ \label{alg:capefm:step_lambda_ag}
    \State Compute $\hat{f}_D(\vect{w}) \gets \sum_{j=0}^J \inprod{\hat{\Lambda}_j}{\bar{\phi}_j}$\\
    \Return $\hat{f}_D(\vect{w})$
    \end{algorithmic}
\end{algorithm}
Our proposed method, $\capefm$, is described in Algorithm \ref{alg:capefm} and is based on the $\cape$ scheme described in Section \ref{sec:cape_details}. We exploit the correlated noise to achieve the same performance of the pooled data case (Lemma \ref{lemma:cape}) in the symmetric decentralized setting under the honest-but-curious model for the sites. Recall our assumption that all parties follow the protocol and the number of colluding sites is not more than $\ceil*{S/3} - 1$. The sites collectively generate the noise $e_j^s$ with entries i.i.d. $\sim \mathcal{N}(0, {\tau_{je}^s}^2)$, according to Algorithm \ref{alg:zero-sum-noise-generation}, such that $\sum_{s=1}^S e_j^s = 0$ holds for all $j \in \{0,\ldots,J\}$. The local sites also generate noise $g_j^s$ with entries i.i.d. $\sim \mathcal{N}(0, {\tau_{jg}^s}^2)$. From each site $s$, we release (or send to the central aggregator): $\hat{\Lambda}_j^s = \Lambda_j^s + e_j^s +  g_j^s$ for all $j \in \{0,\ldots,J\}$. Note that $e_j^s$ and $g_j^s$ are arrays of the same dimension as $\Lambda_j^s$. As described in Section \ref{sec:cape_details}, we have
\begin{align}\label{eqn:capefm_noise_variance}
{\tau_{je}^s}^2 &= \left(1-\frac{1}{S}\right){\tau_j^s}^2,\mbox{ and } {\tau_{jg}^s}^2 = \frac{{\tau_j^s}^2}{S}.
\end{align}
Now, at the aggregator we compute the following quantity
\begin{align*}
\hat{\Lambda}_j &= \frac{1}{S} \sum_{s=1}^S \hat{\Lambda}_j^s = \frac{1}{S} \sum_{s=1}^S \Lambda_j^s + \frac{1}{S} \sum_{s=1}^S g_j^s,
\end{align*}
because $\sum_s e_j^s = 0$. The aggregator then uses these $\{\hat{\Lambda}_j\}$ to compute $\hat{f}_D(\vect{w})$ and release $\hat{\vect{w}}^* = \argmin_\vect{w}\hat{f}_D(\vect{w})$. Privacy of $\capefm$ follows directly from Theorem \ref{thm:cape}. In the symmetric setting (i.e., $N_s = \frac{N}{S}$ and $\tau_j^s = \tau_j$ for all sites $s \in [S]$ and all $j \in \{0, 1, \ldots, J\}$), the noise variance at the aggregator is exactly the same as that of the pooled data scenario (see Lemma \ref{lemma:cape} and Proposition~\ref{prop:low_sensitivity}). Additionally, the performance gain of $\capefm$ over any conventional distributed functional mechanism is given by Proposition~\ref{prop:gain}.

\section{Experimental Results}\label{sec:experimental_results}
In this section, we empirically show the effectiveness of the proposed $\cape$ and $\capefm$ algorithms with applications in a neural network based classifier and a linear regression problem in distributed settings. Our $\cape$ algorithm can improve a distributed computation if the target function has sensitivity satisfying the conditions of Proposition \ref{prop:low_sensitivity}. Additionally, the proposed $\capefm$ algorithm is well-suited for distributed optimization, as we can compute the DP approximate $\hat{f}_D(\vect{w})$ of the loss function and then use any off-the-shelf optimizer. If the function can be represented as \eqref{eqn:f_D_with_Lambda_j}, we can employ the $\capefm$ algorithm for an even better utility. However, finding such a representation may be challenging for large neural networks. We first demonstrate the effectiveness of the $\cape$ algorithm for a neural network based classifier with varying $\epsilon$ and total sample size by comparing against the non-privacy-preserving pooled-data approach $(\nonprivT)$ and the conventional approach ($\conv$) of distributed DP gradient descent~\cite{abadi2016}. Then we demonstrate the effectiveness of the $\capefm$ algorithm using a distributed linear regression problem. We present experimental results to show empirical comparison of performance of the proposed $\capefm$ with the existing $\dpfm$~\cite{zhang2012}, objective perturbation $(\objpert)$~\cite{anand2011} and non-private linear regression $(\nonpriv)$ on pooled-data. Note that both $\dpfm$ and $\objpert$ offer the stronger $\epsilon$-DP and are applied to the pooled-data scenario. We also included the performance variation of a conventional DP distributed scheme with no correlated noise $(\conv)$ and a DP linear regression on local (single site) data $(\local)$. For both the neural network based classifier and the linear regression problem, we consider the symmetric setting (i.e., $N_s = \frac{N}{S}$ and $\tau_s = \tau$) and show the average performance over 10 independent runs.

\subsection{Distributed Neural Network-based Classifier}\label{sec:exp_dist_nn}
\begin{figure*}[t]
  \centering
  \includegraphics[width=1\textwidth]{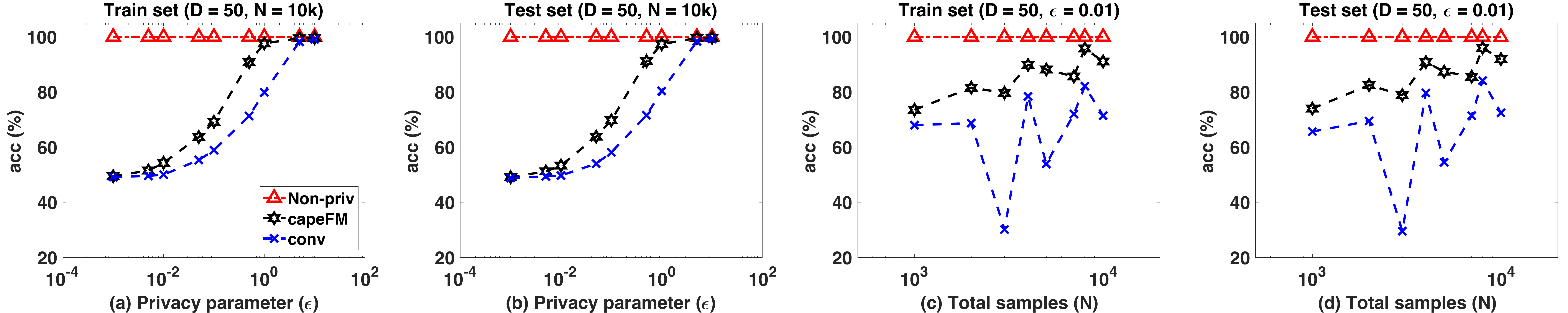}\\
  \vspace{-0.0in}
  \caption{Variation of performance of NN based classifier on synthetic data: (a)--(b) with $\epsilon$ per iteration; (c)--(d) with $N$. Fixed parameter: $S = 4$.}
  \label{fig:all_vs_nn}
\end{figure*}

We evaluate the performance of the proposed $\cape$ scheme on a neural network based classifier to classify between the two classes on a synthetic dataset. The samples from the two classes are random Gaussian vectors with unit variance and the means are separated by 1. We have the same distributed setup as mentioned before: data samples are distributed across $S$ sites. Let $\matr{X}_s \in \mathbb{R}^{D \times N_s}$ be the sample matrix at site $s$ with $N_s$ samples in $D$ dimensions. Let $\vect{y}_s \in \{0, 1\}^{N_s}$ contain the labels for each sample. The global sample matrix and labels vector are given as $\matr{X} = \left[\matr{X}_1, \ldots, \matr{X}_S\right] \in \mathbb{R}^{D\times N}$ and $\vect{y} = \left[\vect{y}_1^\top, \ldots, \vect{y}_S^\top\right]^\top \in \mathbb{R}^N$, where $N = \sum_{s=1}^S N_s$. We consider an $L = 2$ layer neural network and $D=50$, $N = 10k$, $S = 4$. The number of units in layer $l \in \{0, 1, \ldots, L\}$ is denoted by $D^{[l]}$ with $D^{[0]} = D$. We use Rectified Linear Unit (ReLU) activation in the hidden layer and sigmoid activation in the output layer. Our goal is to find the following parameters of the neural network: $\matr{W}^{[1]} \in \mathbb{R}^{D^{[1]} \times D^{[0]}}$, $\vect{b}^{[1]} \in \mathbb{R}^{D^{[1]}}$, $\matr{W}^{[2]} \in \mathbb{R}^{D^{[2]} \times D^{[1]}}$ and $\vect{b}^{[2]}\in \mathbb{R}^{D^{[2]}}$, such that $\texttt{round}\left(\hat{\vect{y}}\right)$ gives the labels of the samples with minimum deviation from $\vect{y}$, where $\hat{\vect{y}} = \texttt{sigmoid}\left(\matr{W}^{[2]}\matr{A}^{[1]} + \vect{b}^{[2]}\right) \mbox{ and } \matr{A}^{[1]} = \texttt{ReLU}\left(\matr{W}^{[1]}\matr{X} + \vect{b}^{[1]}\right)$. We use distributed gradient descent to minimize the empirical average cross-entropy loss. The cross-entropy loss for the $n$-th sample $\vect{x}_n$ is: 
\begin{align}\label{eqn:xentropy_loss}
l(\hat{y}_n, y_n) &= -y_n\log \hat{y}_n - (1-y_n)\log(1-\hat{y}_n).
\end{align}
Each site evaluates the gradients on the local data and sends privacy-preserving approximates of the gradients to the aggregator according to a $\nonprivT$, $\cape$ or $\conv$ schemes. The aggregator then combines the gradient contributions and updates the parameters. For $\cape$, we set $\tau_s = \frac{\Delta}{\epsilon}\sqrt{2\log \frac{1.25}{0.01}}$ for all experiments, where $\Delta$ is the $\mathcal{L}_2$ sensitivity of the corresponding gradient. To measure the utility of the estimated parameters, we use the percent accuracy $\mathrm{acc} = \frac{1}{N_{test}} \sum_{n = 1}^{N_{test}} \mathcal{I}\left(\texttt{round}(\hat{\vect{y}}_{test, n}) = \vect{y}_{test, n}\right) \times 100$, where $\mathcal{I}(\cdot)$ is the indicator function and $\texttt{round}\left(\vect{y}_{test}\right) \in \mathbb{R}^{N_{test}}$ contains the labels of the test-set samples.

\noindent\textbf{Performance Variation with $\epsilon$ and $N$. }We consider the privacy-utility tradeoff first. In Figure \ref{fig:all_vs_nn}(a)--(b), we show the variation of $\mathrm{acc}$ with $\epsilon$ per iteration on the train and test sets for the synthetic dataset, while keeping $N$ and $S$ fixed. We observe that both of the privacy preserving algorithms: $\cape$ and $\conv$, perform better as we increase $\epsilon$. The proposed $\cape$ algorithm performs better than $\conv$ and reaches the performance of $\nonprivT$ even for moderate $\epsilon$ values. Next, we consider the variation of $\mathrm{acc}$ with total sample size $N$, while keeping $\epsilon$ per iteration and $S$ fixed. In Figure \ref{fig:all_vs_nn}(c)--(d), we show the variation of $\mathrm{acc}$ on the train and test sets. As in the case of varying $\epsilon$, we observe that increasing $N$ (hence increasing $N_s$) improves $\mathrm{acc}$ for the privacy-preserving algorithms. We recall that the variance of the noise added for privacy is inversely proportional to $N^2$ -- therefore, availability of more samples results in better accuracy for the same privacy level. Again, we observe that the proposed $\cape$ algorithm outperforms the $\conv$ algorithm. However, when $N$ is quite small, the performance gap between $\cape$ and $\conv$ is less pronounced. In practice, we observe that the gradient descent converges much faster under $\cape$ than the $\conv$, matching the performance of $\nonprivT$ This can be explained by recalling the fact that the $\cape$ algorithm offers the benefit of much less additive noise at the aggregator.

\subsection{Distributed Linear Regression}\label{sec:exp_dist_lr}
\begin{figure*}[t]
  \centering
  \includegraphics[width=1\textwidth]{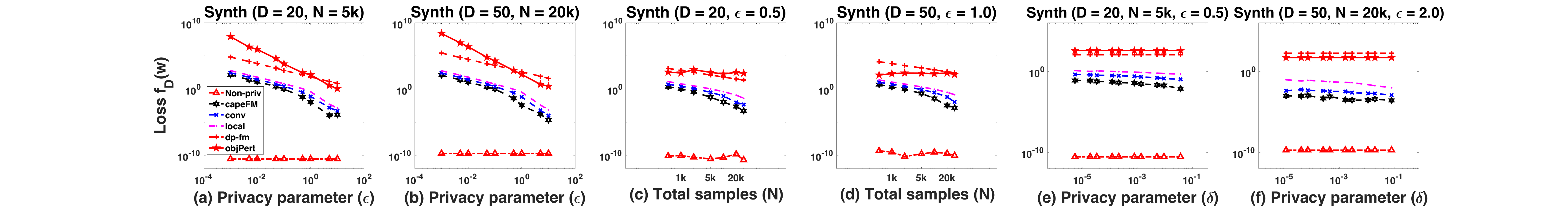}\\
  \vspace{-0.0in}
  \caption{Variation of loss $f_D(\vect{w})$ at $\vect{w}^*$ for synthetic datasets. (a)--(b): with $\epsilon$. (c)--(d): with total samples $N$. (e)--(f): with $\delta$. Fixed param.: $S = 5$.}
  \vspace{-0.0in}
  \label{fig:loss_vs_all_synth}
\end{figure*}

\begin{figure*}[t]
  \centering
  \includegraphics[width=1\textwidth]{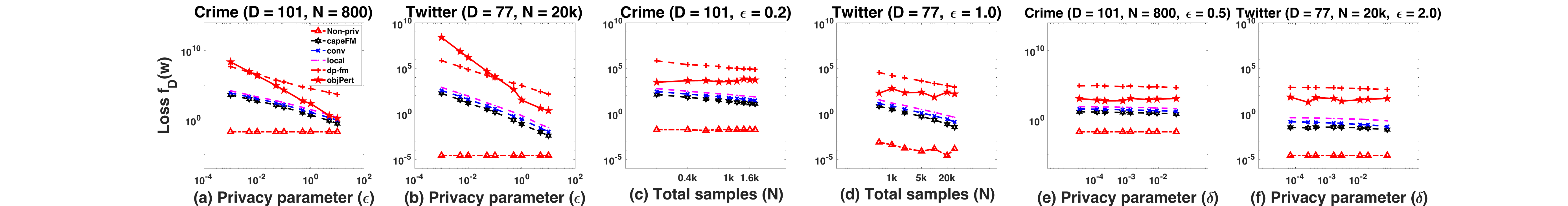}\\
  \vspace{-0.0in}
  \caption{Variation of loss $f_D(\vect{w})$ at $\vect{w}^*$ for two real datasets. (a)--(b): with $\epsilon$. (c)--(d): with total samples $N$. (e)--(f): with $\delta$. Fixed param.: $S = 5$.}
  \vspace{-0.1in}
  \label{fig:loss_vs_all_real}
\end{figure*}

To demonstrate the effectiveness of $\capefm$, we performed experiments on three datasets: a \textit{synthetic} dataset ($D = 20$ or $D = 50$) (Synth) generated with a random $\vect{w}^*$ and random samples $\matr{X}$, the \textit{Communities and Crime} dataset ($D = 101$)~\cite{Lichman:2013} (Crime) and the \textit{Buzz in social media} dataset ($D = 77$)~\cite{Lichman:2013} (Twitter). We refer the reader to~\cite{Lichman:2013} for a detailed description of these real datasets. Now, for each of the datasets, we normalized each feature across the samples to ensure that each feature lies in the range $[-1, 1]$. We also normalized the samples with the maximum $\mathcal{L}_2$ norm in each dataset to ensure $\|\vect{x}_n\|_2 \le 1 \ \ \forall n$. This preprocessing step is not differentially private but can be modified to satisfy privacy at the cost of some utility. For $\capefm$, we set $\tau_j^s = \frac{\Delta_j^s}{\epsilon}\sqrt{2\log \frac{1.25}{10^{-5}}}$ for experiments on the synthetic datasets and $\tau_j^s = \frac{\Delta_j^s}{\epsilon}\sqrt{2\log \frac{1.25}{10^{-3}}}$ for experiments on the real datasets, where $\Delta_j^s$ is the $\mathcal{L}_2$ sensitivity of the corresponding $\Lambda_j$. We use two performance indices for the synthetic dataset. The first one is the empirical loss 
\begin{align}
f_D(\vect{w}) &= \frac{1}{N}\|\vect{y} - \matr{X}\vect{w}\|_2^2,
\end{align}
where $\matr{X}$ and $\vect{y}$ contains all data tuples $(\matr{X}_s, \vect{y}_s)$ from $s \in [S]$. We compute $f_D(\vect{w})$ at the final $\vect{w}$, which is found from minimizing the output function of the algorithm. The second one is closeness to the true $\vect{w}^*$. If the final $\vect{w}$ achieved by minimizing the output function of the algorithm is denoted by $\hat{\vect{w}}^*$ then we define: $\mrm{err}_w = \frac{1}{D} \|\vect{w}^* - \hat{\vect{w}}^*\|_2$. As the true $\vect{w}^*$ is unknown for the real datasets, we only use $f_D(\vect{w})$ as the performance index for Crime and Twitter datasets. The variation of $\mrm{err}_w$ with several parameters on synthetic datasets is shown in Figure~\ref{fig:err_vs_all_synth} in Appendix~\ref{appendix:err_vs_all_synth} in the Supplement.

\noindent\textbf{Dependence on Privacy Parameter $\epsilon$. }First, we explore the trade-off between privacy and utility. We note that, as we employ the Gaussian mechanism, the standard deviation of the added noise is inversely proportional to $\epsilon$ -- bigger $\epsilon$ means higher privacy risk but less noise and thus, better utility. We observe this in our experiments as well. In Figure \ref{fig:loss_vs_all_synth}(a)--(b), we show the variation of $f_D(\vect{w})$ of different algorithms for different values of $\epsilon$ on synthetic data. For this experiment, we kept the number of total samples $N$ and the number of sites $S$ fixed. We show the plots for two different feature dimensions: $D=20$ and $D=50$. For both of the synthetic datasets, we observe that as $\epsilon$ increases (higher privacy risk), the loss $f_D(\vect{w})$ decreases. The proposed $\capefm$ reaches very small $f_D(\vect{w})$ for some parameter choices and clearly outperforms the $\dpfm$, $\objpert$, $\conv$ and $\local$. One of the reasons that $\capefm$ outperforms $\conv$ is the smaller noise variance at the aggregator that we can achieve due to the correlated noise scheme. Moreover, $\capefm$ outperforms $\dpfm$ because $\dpfm$ suffers from a much larger variance at the aggregator (due to the conservative sensitivity computation of $\lambda_{n\phi}$). On the other hand, $\objpert$ also entails addition of noise with large variance as the sensitivity of the optimal $\vect{w}^*$ is large (to be exact, the sensitivity is 2). Achieving better performance than $\local$ is intuitive because including the information from multiple sites to estimate a population parameter always results in better performance than using the data from a single site only. Additionally, we observe that for datasets with lower dimensional samples, we can use smaller $\epsilon$ (i.e., to guarantee lower privacy risk) for the same utility. In Figure \ref{fig:loss_vs_all_real}(a)--(b), we show the variation of $f_D(\vect{w})$ with $\epsilon$ for the Crime and the Twitter datasets. We observe similar variation characteristic of $f_D(\vect{w})$ for the real datasets as we observed for the synthetic datasets. Note that, for real datasets, we chose larger $\tau_s$ values than synthetic datasets to achieve the similar utility. 

\noindent\textbf{Dependence on Total Sample Size $N$. }Next, we investigate the variation in performance with the total sample size $N$. Intuitively, it should be easier to guarantee smaller privacy risk $\epsilon$ and higher utility, when $N$ is large. Figure \ref{fig:loss_vs_all_synth}(c)--(d) show how $f_D(\vect{w})$ decreases as a function of $N$ on synthetic data. The variation with $N$ reinforces the results seen earlier with variation of $\epsilon$. For a fixed $\epsilon$ and $S$, the $f_D(\vect{w})$ decreases as we increase $N$. For sufficiently large $N$ and $\epsilon$, $f_D(\vect{w})$ will reach that of the non-private pooled case $(\nonpriv)$. We observe a sharper decrease in $f_D(\vect{w})$ for lower-dimensional datasets. In Figure \ref{fig:loss_vs_all_real}(c)--(d), we show the variation of $f_D(\vect{w})$ with $N$ for the Crime and the Twitter datasets. Again, we observe similar variation of $f_D(\vect{w})$ for the real datasets as we observed for the synthetic datasets. 

\noindent\textbf{Dependence on Privacy Parameter $\delta$. }Note that, the proposed $\capefm$ algorithm guarantees $(\epsilon, \delta)$ differential privacy where $(\epsilon, \delta)$ satisfy the relation $\delta = 2\frac{\sigma_z}{\epsilon - \mu_z}\phi\left(\frac{\epsilon - \mu_z}{\sigma_z}\right)$. Recall that $\delta$ can be considered as the probability that the algorithm releases the private information without guaranteeing privacy. Therefore, we want this to be as small as possible. However, smaller $\delta$ also dictates larger noise variance. We explore the variation of performance with different $\delta$ with a fixed number of colluding sites $S_C = \ceil*{\frac{S}{3}} - 1$. In Figure \ref{fig:loss_vs_all_synth}(e)--(f), we show how $f_D(\vect{w})$ varies with varying $\delta$ on synthetic data. We observe that if $N$ and $\delta$ are too small, the proposed $\capefm$, $\conv$ and $\local$ algorithms perform poorly. However, our $\capefm$ algorithm can achieve very good utility for moderate $N$ and $\delta$ values, easily outperforming the other DP algorithms. We observe in Figure \ref{fig:loss_vs_all_real}(e)--(f) similar characteristic for the two real datasets, although the variation of $f_D(\vect{w})$ is not as pronounced as for the synthetic datasets. Recall that the $\dpfm$ and the $\objpert$ algorithms offer pure $\epsilon$-DP guarantee and, therefore, do not vary with $\delta$.

\section{Conclusions}\label{sec:conclusion}
This paper proposes a novel protocol, $\cape$, for distributed differentially private computations. $\cape$ is best suited for applications in which private data must be held locally, e.g. in health care research with legal and ethical limitations on the degree of sharing the ``raw'' data. $\cape$ can greatly improve the privacy-utility tradeoff when (a) all parties follow the protocol and (b) the number of colluding sites is not more than $\ceil*{S/3} - 1$. Our proposed $\cape$ protocol is based on an estimation-theoretic analysis of the noise addition process for differential privacy and therefore, provides different guarantees than cryptographic approaches such as SMC. In addition to $\cape$, we proposed a new algorithm for differentially private computation of functions in distributed settings. Our $\capefm$ algorithm can be employed to compute any continuous and differentiable function. As mentioned before, approximation of the empirical average cost is required in many distributed optimization problems. We proposed an improved functional mechanism with a new way to compute the associated sensitivities. We analytically showed that the proposed approach for computing the sensitivities offers much less additive noise for two common regression problems -- linear regression and logistic regression. Our proposed $\capefm$ can approximate the privacy-preserving empirical average cost such that we can achieve the same utility level as the pooled data scenario in certain regimes. We achieve this by employing the $\cape$ protocol. We empirically compared the performance of the proposed algorithms with those of existing and conventional algorithms for a neural-network based classification problem and a distributed linear regression problem. We varied privacy parameters and relevant dataset (synthetic and real) parameters and showed that the proposed algorithms outperformed the existing and conventional algorithms comfortably -- matching the performance of the non-private algorithm for some parameter choices. In general, the proposed algorithms offered very good utility indicating that meaningful privacy can be attained without losing much performance by the virtue of algorithm design. A very interesting future work could be to extend the $\cape$ framework to fit the optimal Staircase Mechanism~\cite{geng2015} for differential privacy. Another possible direction is to extend $\cape$ to be employable in arbitrary tree-structured networks.

\section*{Acknowledgements}
The work of the authors was supported by the the NIH under award 1R01DA040487-01A1, the NSF under award CCF-1453432, and by DARPA and SSC Pacific under contract No. N66001-15-C-4070.

\bibliography{refs.bib}
\bibliographystyle{IEEEtran}

\clearpage

\section*{Appendix}
\appendix
\section{Communication Overhead} \label{appendix:cape_comm}
The conventional $D$-dimensional averaging needs only one message from each site, thus $SD$ or $\Theta(SD)$ is the communication complexity. Our $\cape$ scheme employs the $\secureagg$ protocol to compute the zero-sum noise. The $\secureagg$ protocol~\cite{Bonawitz17} entails an $O(S + D)$ overhead for each site and $O(S^2 + SD)$ for the server/aggregator. The rest of our scheme requires $\Theta(D)$ and $\Theta(SD)$ communication overheads for the sites and  the aggregator, respectively. On the other hand, the scheme proposed in~\cite{Mikko17} has a communication cost proportional to $(S+1)DM$ or $\Theta(SDM)$, where $M$ is the number of compute nodes. Goryczka et al.~\cite{SlawomirMPC17} compared several secret sharing, homomorphic encryption and perturbation-based secure sum aggregation and showed their communication complexities. Except for the secret sharing approach (which requires $O(S^2)$ overhead), the other approaches are $O(S)$ in communication complexity. A comparison of communication overhead for different algorithms are shown in Table \ref{tbl:commn_cost}.

\begin{table}[t]
\caption{Comparison of communication overhead}
\label{tbl:commn_cost}
\centering
\begin{tabular}{lll}
  \toprule
  \textbf{Algorithm} & \textbf{Site} & \textbf{Aggregator}\\
  \midrule
  $\cape$ & $O(S + D)$ & $O(S^2 + SD)$ \\
  Heikkil\"{a} et al.~\cite{Mikko17} & $\Theta(DM)$ & $\Theta(SDM)$\\
  Bonawitz et al.~\cite{Bonawitz17} & $O(S + D)$ & $O(S^2 + SD)$\\
  \bottomrule
\end{tabular}
\end{table}

\section{Empirical Comparison of $\delta$ and $\delta_\mrm{conv}$}\label{appendix:eff_delta}
\begin{figure}[t]
  \centering
  \includegraphics[width=1\columnwidth]{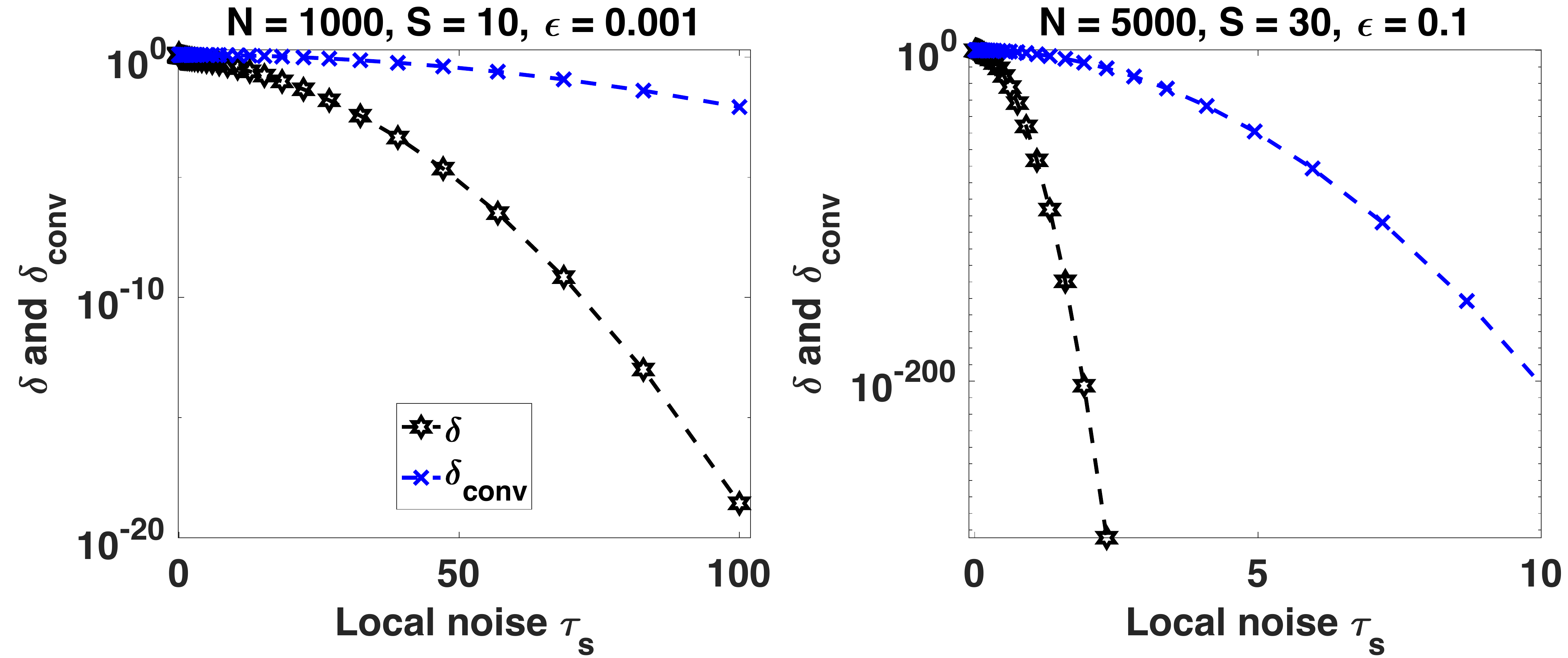}\\
  \vspace{-0.0in}
  \caption{Variation of $\delta$ and $\delta_\mrm{conv}$ with $\tau_s$ for different values of $S$ and $\epsilon$}
  \label{fig:delta_vs_eff_delta}
\end{figure}
Recall that the $\cape$ protocol guarantees $(\epsilon, \delta)$-DP with $\delta = 2\frac{\sigma_z}{\epsilon - \mu_z}\phi\left(\frac{\epsilon - \mu_z}{\sigma_z}\right)$. We claim that this $\delta$ guarantee is much better than the $\delta$ guarantee in the conventional distributed DP scheme. As $\delta$ is an implicit function of $S,\ S_C$ and $\tau_s^2$, we experimentally compare $\delta$ with $\delta_\mrm{conv}$, where $\delta_\mrm{conv}$ is the smallest $\delta$ guarantee we can afford in the conventional distributed DP scheme to achieve the same noise variance as the pooled-data scenario for a given $\epsilon$. We plot $\delta$ and $\delta_\mrm{conv}$ against different $\tau_s$ values for $S_C = \ceil*{\frac{S}{3}} - 1$ and different combinations of $\epsilon$ and $S$ in Figure \ref{fig:delta_vs_eff_delta}. We observe from the figure that $\delta$ is always smaller than $\delta_\mrm{conv}$. That is, for achieving the same noise level at the aggregator output (and therefore the same utility) as the pooled data scenario, we are ensuring a much better privacy guarantee by employing the $\cape$ scheme over the conventional approach. 

\section{Effect of Fewer Colluding Sites}\label{appendix:fewer_colluding_sites}
As mentioned in Section~\ref{sec:cape_utility}, we are interested in how the $\delta$ and $\delta_\mrm{conv}$ vary with weaker collusion assumption (i.e., fewer colluding sites). To that end, we vary the fraction $\frac{S_C}{S}$ and plot the resulting $\delta$ and $\delta_\mrm{conv}$ for different combinations of $\epsilon$, $S$ and $\tau_s$ in Figure \ref{fig:Sc_vs_eff_delta}. Again, we observe that $\delta$ is always smaller than $\delta_\mrm{conv}$. That is, we are ensuring a much better privacy guarantee by employing the $\cape$ scheme over the conventional approach for achieving the same noise level at the aggregator output (and therefore the same utility) as the pooled data scenario. 
\begin{figure}[t]
  \centering
  \includegraphics[width=1\columnwidth]{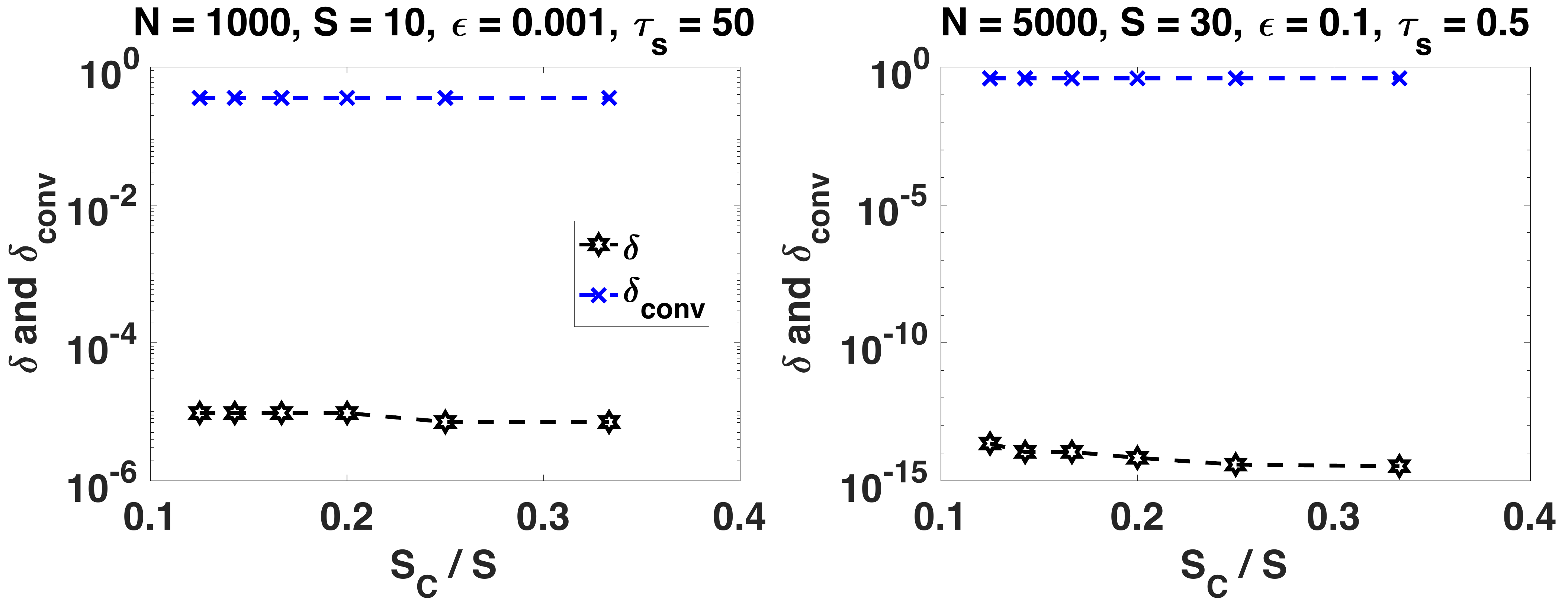}\\
  \vspace{-0.0in}
  \caption{Variation of $\delta$ and $\delta_\mrm{conv}$ with $\frac{S_C}{S}$, for different values of $S$ and $\epsilon$}
  \label{fig:Sc_vs_eff_delta}
\end{figure}

\section{Proof of Lemma~\ref{lemma:cape}}\label{appendix:cape_lemma}
\begin{proof}
We prove the lemma according to~\cite{imtiaz2018}. Recall that in the pooled data scenario, the sensitivity of the function $f(\vect{x})$ is $\frac{1}{N}$, where $\vect{x} = \left[\vect{x}_1,\ldots, \vect{x}_S\right]$. Therefore, to approximate $f(\vect{x})$ satisfying $(\epsilon, \delta)$ differential privacy, we need to have additive Gaussian noise standard deviation at least $\tau_\mrm{pool} = \frac{1}{N\epsilon}\sqrt{2\log\frac{1.25}{\delta}}$. Next, consider the distributed data setting (as in Section~\ref{sec:problem_formulation}) with local noise standard deviation given by
\begin{align*}
\tau_s &= \frac{1}{N_s\epsilon}\sqrt{2\log\frac{1.25}{\delta}} = \frac{S}{N\epsilon}\sqrt{2\log\frac{1.25}{\delta}} = \tau
\end{align*}
We observe $\tau_\mrm{pool} = \frac{\tau_s}{S} \implies \tau_\mrm{pool}^2 = \frac{\tau^2}{S^2}$. We will now show that the $\cape$ algorithm will yield the same noise variance of the estimator at the aggregator. Recall that at the aggregator we compute $a_\mrm{cape} = \frac{1}{S} \sum_{s=1}^S \hat{a}_s = \frac{1}{N} \sum_{n=1}^{N} x_n + \frac{1}{S} \sum_{s=1}^S g_s$. The variance of the estimator $a_\mrm{cape}$ is:
\begin{align*}
\tau_\mrm{cape}^2 \triangleq S\cdot \frac{\tau_g^2}{S^2} = \frac{\tau_g^2}{S} = \frac{\tau^2}{S^2},
\end{align*}
which is exactly the same as the pooled data scenario. Therefore, the $\cape$ algorithm allows us to achieve the same additive noise variance as the pooled data scenario, while satisfying $(\epsilon, \delta)$ differential privacy at the sites and for the final output from the aggregator, where $(\epsilon, \delta)$ satisfy the relation $\delta = 2\frac{\sigma_z}{\epsilon - \mu_z}\phi\left(\frac{\epsilon - \mu_z}{\sigma_z}\right)$.
\end{proof}

\section{Performance Gain of $\cape$}\label{appendix:perf_gain}
\begin{proof}[Proof of Proposition~\ref{prop:gain}]
The local noise variances are $\{\tau_s^2\}$ for $s \in [S]$. In the conventional distributed DP scheme, we compute the following at the aggregator:
\begin{align*}
a_\mrm{conv} &= \frac{1}{S}\sum_{s=1}^S a_s + \frac{1}{S}\sum_{s=1}^S e_s.
\end{align*}
The variance of the estimator is:
\begin{align*}
\tau_\mrm{conv}^2 &= \sum_{s=1}^s \frac{\tau_s^2}{S^2} = \frac{1}{S^2} \sum_{s=1}^s \tau_s^2.
\end{align*}
In the $\cape$ scheme, we compute the following quantity at the aggregator:
\begin{align*}
a_\mrm{cape} &= \frac{1}{S}\sum_{s=1}^S a_s + \frac{1}{S}\sum_{s=1}^S e_s + \frac{1}{S}\sum_{s=1}^S g_s.
\end{align*}
The variance of the estimator is:
\begin{align*}
\tau_\mrm{cape}^2 &= \sum_{s=1}^s \frac{\tau_g^2}{S^2} = \frac{1}{S^3} \sum_{s=1}^s \tau_s^2.
\end{align*}
Therefore, the gain of the $\cape$ scheme over conventional distributed DP approach is
\begin{align*}
G &= \frac{\tau_\mrm{conv}^2}{\tau_\mrm{cape}^2} = S,
\end{align*}
which completes the proof.
\end{proof}

\section{Proof of Proposition~\ref{prop:low_sensitivity}}\label{appendix:low_sensitivity}
We start with reviewing some definitions and lemmas~\cite[Proposition C.2]{majorization} necessary for the proof. 

\begin{Def}[Majorization]
Consider two vectors $\vect{a} \in \mathbb{R}^S$ and $\vect{b} \in \mathbb{R}^S$ with non-increasing entries (i.e., $a_i \geq a_j$ and $b_i \geq b_j$ for $i < j$). Then $\vect{a}$ is majorized by $\vect{b}$, denoted $\vect{a} \prec \vect{b}$, if and only if the following holds:
\begin{align*}
\sum_{s=1}^S a_s &= \sum_{s=1}^S b_s \mbox{ and } \sum_{s=1}^J a_s \leq \sum_{s=1}^J b_s\ \forall J \in [S].
\end{align*}
\end{Def}

\noindent Consider $\vect{n}_\mrm{sym} \triangleq \frac{N}{S}[1,\ldots,1] \in \mathbb{R}^S$ for some positive $N$. Then any vector $\vect{n} = [N_1,\ldots,N_S] \in \mathbb{R}^S$ with non-increasing entries and $\sum_{s=1}^S |N_s| = N$ majorizes $\vect{n}_\mrm{sym}$, or $\vect{n}_\mrm{sym} \prec \vect{n}$.

\begin{Def}[Schur-convex functions]\label{def: SchurCVX}
The function $K: \mathbb{R}^{S} \mapsto \mathbb{R}$ is Schur-convex if for all $\vect{a} \prec \vect{b} \in \mathbb{R}^{S}$ we have $K(\vect{a}) \leq K(\vect{b})$.
\end{Def}

\begin{lemma}\label{lem: schurconx}
If $K$ is symmetric and convex, then $K$ is Schur-convex. The converse does not hold.
\end{lemma}

\section{Proof of Theorem~\ref{thm:fm}}\label{appendix:thm:fm}
\begin{proof} The proof of Theorem \ref{thm:fm} follows from the fact that the function $\hat{f}_D(\vect{w})$ depends on the data samples only through $\{\hat{\Lambda}_j\}$.  The computation of $\{\hat{\Lambda}_j\}$ is $(\epsilon, \delta)$-differentially private by the Gaussian mechanism~\cite{dwork2006,dwork2013algorithmic}, so the release of $\hat{f}_D(\vect{w})$ satisfies $(\epsilon, \delta)$-differential privacy. One way to rationalize this is to consider that the probability of the event of selecting a particular set of $\{\hat{\Lambda}_j\}$ is the same as the event of formulating a function $\hat{f}_D(\vect{w})$ with that set of $\{\hat{\Lambda}_j\}$. Therefore, it suffices to consider the joint density of the $\{\hat{\Lambda}_j\}$ and find an upper bound on the ratio of the joint densities of the $\{\hat{\Lambda}_j\}$ under $\mathbb{D}$ and $\mathbb{D}'$. As we employ the Gaussian mechanism to compute $\{\hat{\Lambda}_j\}$, the ratio is upper bounded by $\exp(\epsilon)$ with probability at least $1-\delta$. Therefore, the release of $\hat{f}_D(\vect{w})$ satisfies $(\epsilon, \delta)$-differential privacy. Furthermore, differential privacy is closed under post processing. Therefore, the computation of the minimizer $\hat{\vect{w}}^* = \argmin_\vect{w} \hat{f}_D(\vect{w})$ also satisfies $(\epsilon, \delta)$-differential privacy.
\end{proof}

\section{Improved Functional Mechanism for Logistic Regression}\label{appendix:log_reg}
\begin{figure*}[t]
  \centering
  \includegraphics[width=1\textwidth]{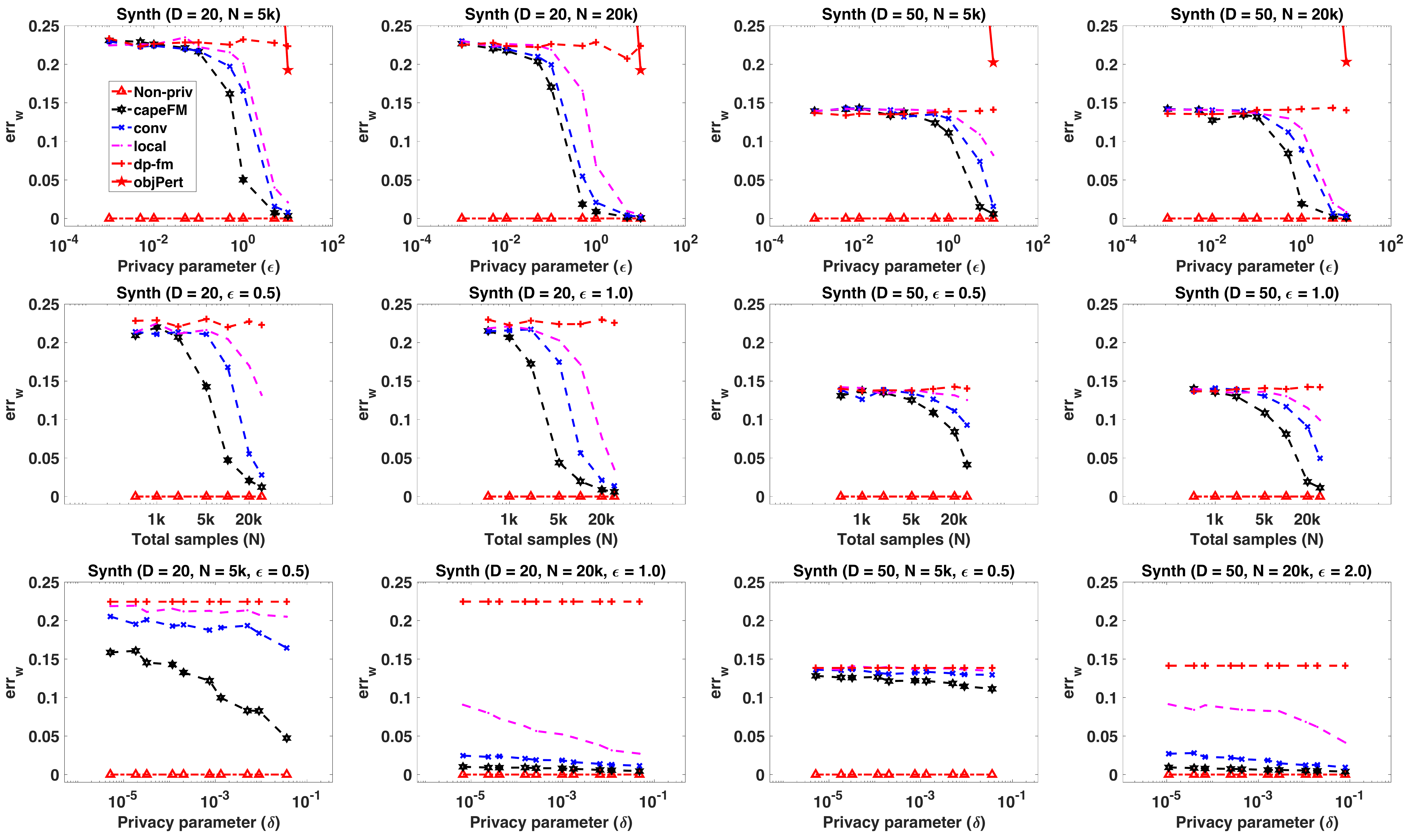}\\
  \vspace{-0.0in}
  \caption{Variation of $\mrm{err}_w$ for synthetic datasets. Top-row: with $\epsilon$. Mid-row: with total samples $N$. Bottom-row: with $\delta$. Fixed parameter: $S = 5$.}
  \label{fig:err_vs_all_synth}
\end{figure*}

In this section, we demonstrate the proposed improved functional mechanism for a logistic regression problem in the centralized setting. As in Section~\ref{sec:fm_example}, we show that Algorithm \ref{alg:fm} achieves much better utility at the price of a weaker privacy guarantee ($(\epsilon, \delta)$-DP vs $\epsilon$-DP). The main reason for this performance improvement is due to defining the sensitivities of $\Lambda_j$ separately for each $j$, instead of using an uniform conservative upper-bound $\Delta^\mrm{dp-fm}$ (as in~\cite{zhang2012}).

We have a dataset $\mathbb{D}$ with $N$ samples. Each sample is a tuple $(\vect{x}_n, y_n)$, where $\vect{x}_n \in \mathbb{R}^D$ is the feature vector and $y_n \in \{0, 1\}$ is the label. Without loss of generality, we assume that $\|\vect{x}_n\|_2 \leq 1$. We want to find a vector $\vect{w} \in \mathbb{R}^D$ such that $\mathcal{I}(\vect{x}_n^\top\vect{w} \geq 0)$ gives the label $y_n$ for all $n\in [N]$, where $\mathcal{I}(\cdot)$ denotes the indicator function. The cost function due to each sample and a particular $\vect{w}$ is $f:\mathbb{R}^D \times \mathbb{R}^D \mapsto \mathbb{R}$ and is defined as the logistic loss:
\begin{align}\label{eqn:logistic_loss_sample}
f(\vect{x}_n, \vect{w}) &= \log\left(1 + \exp\left(\vect{x}_n^\top\vect{w}\right)\right) - y_n \vect{x}_n^\top\vect{w}.
\end{align}
The empirical average cost function is defined as
\begin{align}\label{eqn:logistic_loss_avg}
f_D(\vect{w}) 	&= \frac{1}{N} \sum_{n=1}^N\log\left(1 + \exp\left(\vect{x}_n^\top\vect{w}\right)\right) - y_n \vect{x}_n^\top\vect{w}.
\end{align}
It is not readily apparent how to express this logistic loss function in the form of \eqref{eqn:empirical_cost}. Zhang et al.~\cite{zhang2012} derived an approximate polynomial form of $f_D(\vect{w})$, denoted by $\tilde{f}_D(\vect{w})$, using a Taylor series expansion. Each function $f(\vect{x}_n, \vect{w})$ can be represented as
\begin{align*}
f(\vect{x}_n, \vect{w}) &= \sum_{m=1}^M f_m\left(g_m\left(\vect{x}_n, \vect{w}\right)\right)
\end{align*}
for some functions $f_m(\cdot)$ and $g_m(\cdot)$, where each $g_m(\cdot)$ is a monomial of $\{w_d\}$. Using the Taylor expansion of $f_m(\cdot)$ around any real $z_m$, we have
\begin{align*}
f(\vect{x}_n, \vect{w}) &= \sum_{m=1}^M \sum_{j=0}^\infty \frac{f_m^{(j)}(z_m)}{j!} \left(g_m\left(\vect{x}_n, \vect{w}\right) - z_m\right)^j,
\end{align*}
where $f_m^{(j)}(\cdot)$ is the $j$-th derivative of $f_m(\cdot)$. Now, for the logistic loss given in \eqref{eqn:logistic_loss_sample}, we have~\cite{zhang2012} that
\begin{align*}
g_1(\vect{x}_n, \vect{w}) &= \vect{x}_n^\top \vect{w}\\
f_1(z) &= \log(1+\exp(z))\\
g_2(\vect{x}_n, \vect{w}) &= y_n\vect{x}_n^\top \vect{w}\\
f_2(z) &= -z.
\end{align*}
Therefore, the empirical average cost can be written as
\begin{align*}
f_D(\vect{w}) 	&= \frac{1}{N} \sum_{n=1}^N \sum_{m=1}^2 \sum_{j=0}^\infty \frac{f_m^{(j)}(z_m)}{j!} \left(g_m\left(\vect{x}_n, \vect{w}\right) - z_m\right)^j.
\end{align*}
Zhang et al.~\cite{zhang2012} approximated this infinite sum for $j = 0, 1, 2$ and showed analysis for the excess error of the approximation. For $J=2$, the approximation error is a small constant. Now, for $m = 1,2$ and $j = 0, 1, 2$, the approximate empirical average cost function can be written as
\begin{align*}
\tilde{f}_D(\vect{w}) &= \frac{1}{N} \sum_{n=1}^N \sum_{m=1}^2 \sum_{j=0}^2 \frac{f_m^{(j)}(z_m)}{j!} \left(g_m\left(\vect{x}_n, \vect{w}\right) - z_m\right)^j.
\end{align*}
Using the expressions of $f_m^{(j)}$ and some simple algebra, it can be shown that
\begin{align*}
\tilde{f}_D(\vect{w}) &=\log2 + \sum_{d=1}^D\left(\frac{1}{N}\sum_{n=1}^N\left(\frac{1}{2} - y_n\right)x_{nd}\right)w_d \\
& + \sum_{d_1=1}^D\sum_{d_2=1}^D\left(\frac{1}{8N}\sum_{n=1}^N x_{nd_1}x_{nd_2}\right)w_{d_1}w_{d_2},
\end{align*}
where $x_{pq}$ is the $q$-th element of the $p$-th sample vector. Using our definitions of $\Lambda_j$ as before, we have
\begin{align*}
\Lambda_0 &= \log2\\
\Lambda_1 &= \frac{1}{N}	\begin{bmatrix}
    							\sum_{n=1}^N \left(\frac{1}{2} - y_n\right) x_{n1} \\
                				\vdots \\
    							\sum_{n=1}^N \left(\frac{1}{2} - y_n\right) x_{nD}
  							\end{bmatrix} = \frac{\left(\vect{1}^\top \mrm{diag}\left(\frac{1}{2} - \vect{y}\right) \matr{X}\right)^\top}{N} \\
\Lambda_2 &= \frac{1}{8N} 	\begin{bmatrix}
											\sum_{n=1}^N x_{n1}^2 & \cdots & \sum_{n=1}^N x_{n1}x_{nD}\\
                                				\vdots					& \ddots & \vdots \\
                                				\sum_{n=1}^N x_{nD}x_{n1} & \cdots & \sum_{n=1}^N x_{nD}^2
											\end{bmatrix} = \frac{\matr{X}^\top\matr{X}}{8N},
\end{align*}
where $\vect{1}$ is an $N$-dimensional vector of all 1's and $\mrm{diag}\left(\frac{1}{2} - \vect{y}\right)$ is an $N\times N$ diagonal matrix with the elements of $\frac{1}{2} - \vect{y}$ as the diagonal entries. If we express the $\{\bar{\phi}_j\}$ as in the linear regression example (Section~\ref{sec:fm_example}), we can write $\tilde{f}_D(\vect{w})$ as
\begin{align}
\tilde{f}_D(\vect{w}) 	&= \sum_{j=0}^2 \inprod{\Lambda_j}{\bar{\phi}_j}.
\end{align}
Now, we focus on finding the sensitivities of $\{\Lambda_j\}$. Let us consider a neighboring dataset $\mathbb{D}'$ which contains the same tuples as $\mathbb{D}$, except for the last one, i.e. $(\vect{x}_N', y_N')$. We have
\begin{align*}
\left|\Lambda_0^\mathbb{D} - \Lambda_0^{\mathbb{D}'}\right| &= 0 \triangleq \Delta_0.
\end{align*}
Next, we observe
\begin{align*}
\left\|\Lambda_1^\mathbb{D} - \Lambda_1^{\mathbb{D}'}\right\|_2 &= \frac{1}{N} \left\|\left(\frac{1}{2} - y_N\right)\vect{x}_N - \left(\frac{1}{2} - y_N'\right)\vect{x}_N'\right\|_2 \\
																&\leq \frac{2}{N} \left\|\left(\frac{1}{2} - y_N\right)\vect{x}_N\right\|_2 \\
                         	                                   &\leq \frac{2}{N} \frac{1}{2} \|\vect{x}_N\|_2 \leq \frac{1}{N} \triangleq \Delta_1,
\end{align*}
where we used the inequality $\left\|\left(\frac{1}{2} - y_N\right)\vect{x}_N\right\|_2 \leq \frac{1}{2} \|\vect{x}_N\|_2 \leq \frac{1}{2}$ because, $y_n \in \{0, 1\}$ and $\|\vect{x}_n\|_2 \leq 1$. Finally, we observe
\begin{align*}
\left\|\Lambda_2^\mathbb{D} - \Lambda_2^{\mathbb{D}'}\right\|_F &= \frac{1}{8N}\left\|\matr{X}^\top\matr{X} - {\matr{X}'}^\top\matr{X}'\right\|_F \\
							&= \frac{1}{8N}\left\|\vect{x}_N\vect{x}_N^\top - \vect{x}_N'{\vect{x}_N'}^\top\right\|_F \leq \frac{\sqrt{2}}{8N} \triangleq \Delta_2,
\end{align*}
where the last inequality follows from realizing that the $D\times D$ symmetric matrix $\vect{x}_N\vect{x}_N^\top - \vect{x}_N'{\vect{x}_N'}^\top$ is at-most rank-2 and $\left\|\vect{x}_N\vect{x}_N^\top - \vect{x}_N'{\vect{x}_N'}^\top\right\|_F \leq \sqrt{2}$, as shown before (Section~\ref{sec:fm_example}). Now that we have computed the $\mathcal{L}_2$-sensitivities of $\{\Lambda_j\}$, we can compute $\{\hat{\Lambda}_j\}$ and thus, the $(\epsilon, \delta)$-differentially private approximation of $\tilde{f}_D(\vect{w})$ following Algorithm \ref{alg:fm} ($\tilde{f}_D(\vect{w})$ is the input to Algorithm \ref{alg:fm} for logistic regression). Note that as with linear regression, the sensitivity computation technique for $\epsilon$-differential privacy~\cite{zhang2012} for each entry of $\Lambda_j$ would be $\Delta^\mrm{dp-fm} = \frac{1}{N}\left(\frac{D^2}{4} + 3D\right)$, which is orders of magnitude larger than $\Delta_j$ for any $D > 1$ and for all $j \in \{0,\ldots,J\}$. Therefore, with the sensitivity computation proposed in this paper, we can achieve $\hat{f}_D(\vect{w})$ from $\tilde{f}_D(\vect{w})$ with much less noise. This would certainly result in a more accurate privacy-preserving estimate of the optimal $\hat{\vect{w}}^* = \argmin_\vect{w}\hat{f}_D(\vect{w})$. However, one cost of the performance improvement, for both linear and logistic regression applications, is the weakening of the privacy guarantee from $\epsilon$-differential privacy to $(\epsilon, \delta)$-differential privacy.

\section{Distributed Linear Regression -- Variation of $\mrm{err}_w$}\label{appendix:err_vs_all_synth}
\noindent\textbf{Dependence on Privacy Parameter $\epsilon$. }In the top-row of Figure \ref{fig:err_vs_all_synth}, we show the variation of $\mrm{err}_w$ of different algorithms for different values of $\epsilon$ on synthetic data. For this experiment, we kept the number of total samples $N$ and the number of sites $S$ fixed. We show the plots for two different feature dimensions: $D=20$ and $D=50$, each with two different sample sizes. For both of the synthetic datasets, we observe that as $\epsilon$ increases (higher privacy risk), the $\mrm{err}_w$ decreases. The proposed $\capefm$ reaches very small $\mrm{err}_w$ for some parameter choices and clearly outperforms the $\dpfm$, $\objpert$, $\conv$ and $\local$. One of the reasons that $\capefm$ outperforms $\conv$ is the smaller noise variance at the aggregator that we can achieve due to the correlated noise scheme. Moreover, $\capefm$ outperforms $\dpfm$ because $\dpfm$ suffers from a much larger variance at the aggregator (due to the conservative sensitivity computation of $\lambda_{n\phi}$). On the other hand, $\objpert$ also entails addition of noise with large variance as the sensitivity of the optimal $\vect{w}^*$ is large (to be exact, the sensitivity is 2). Achieving better performance than $\local$ is intuitive because including the information from multiple sites to estimate a population parameter always results in better performance than using the data from a single site only. Additionally, we observe that for datasets with lower dimensional samples, we can use smaller $\epsilon$ (i.e., to guarantee lower privacy risk) for the same utility.

\noindent\textbf{Dependence on Total Sample Size $N$. }Next, we investigate the variation in performance with the total sample size $N$. The middle-row of Figure \ref{fig:err_vs_all_synth} shows how $\mrm{err}_w$ decreases as a function of total sample size $N$ on synthetic data. The variation with $N$ reinforces the results seen earlier with variation of $\epsilon$. For a fixed $\epsilon$ and $S$, $\mrm{err}_w$ decreases as we increase $N$. For sufficiently large $N$ and $\epsilon$, $\mrm{err}_w$ will reach that of the non-private pooled case $(\nonpriv)$. Again, we observe a sharper decrease in $\mrm{err}_w$ for lower-dimensional datasets. Note that, for the synthetic datasets, the error $\mrm{err}_w$ for $\objpert$ is too large to show on the same scale as other algorithms. That is why the $\mrm{err}_w$ curves for $\objpert$ do not appear in Figure \ref{fig:err_vs_all_synth} (middle-row).

\noindent\textbf{Dependence on Privacy Parameter $\delta$. }In the bottom-row of Figure \ref{fig:err_vs_all_synth}, we show how $\mrm{err}_w$ varies with varying $\delta$ on synthetic data. We observe that if $N$ and $\delta$ are too small, the proposed $\capefm$, $\conv$ and $\local$ algorithms perform poorly. However, the proposed $\capefm$ algorithm can achieve very good utility for moderate $N$ and $\delta$ values, easily outperforming the other differentially-private algorithms. Recall that the $\dpfm$ and the $\objpert$ algorithms offer pure $\epsilon$-differential privacy and, therefore, do not vary with $\delta$. Again we observe that, for the synthetic datasets, the error $\mrm{err}_w$ for $\objpert$ is too large and do not appear in Figure \ref{fig:err_vs_all_synth} (bottom row).

\end{document}